\documentclass{fdsOF} 
\usepackage{amsmath}
\usepackage{paralist}
\usepackage[misc]{ifsym}
\usepackage{epsfig} 
\usepackage{epstopdf} 
\usepackage[colorlinks=true]{hyperref}
\hypersetup{urlcolor=blue, citecolor=red}
\allowdisplaybreaks

\def\currentvolume{X}
 \def\currentissue{X}
  \def\currentyear{200X}
   \def\currentmonth{XX}
    \def\ppages{X-XX}
     \def\DOI{10.48550/arXiv.2311.17225}

\usepackage{latexsym}
\usepackage{amssymb}
\usepackage{amsthm}

\newtheorem{theorem}{Theorem}[section]
\newtheorem{corollary}[theorem]{Corollary}

\newtheorem{lemma}[theorem]{Lemma}
\newtheorem{proposition}[theorem]{Proposition}

\theoremstyle{definition}
\newtheorem{definition}[theorem]{Definition}
\newtheorem{remark}[theorem]{Remark}
\newtheorem{assumption}[theorem]{Assumption}
\newtheorem{example}[theorem]{Example}

\title[Sparse joint shift in multinomial classification]
{Sparse joint shift in multinomial classification} 

\author[Dirk Tasche]{}


\begin{document}
\maketitle

\centerline{\scshape
Dirk Tasche$^{{\href{mailto:dirk.tasche@gmx.net}{\textrm{\Letter}}}}$%
}

\medskip

{\footnotesize
 \centerline{Independent scholar, Switzerland}
} 

%

\bigskip

 \centerline{\ }


\begin{abstract}
Sparse joint shift (SJS) was recently proposed as a tractable model for
general dataset shift which may cause changes to the marginal distributions of features and
labels as well as the posterior probabilities and the class-conditional
feature distributions. Fitting SJS for a target dataset without label observations
may produce valid predictions of labels and estimates of class prior probabilities.
We present new results on the transmission of SJS from sets of features to larger sets of features, a conditional
correction formula for the class posterior probabilities under the target distribution,
identifiability of SJS, and the relationship between SJS and covariate shift.
In addition, we point out
inconsistencies in the algorithms which were proposed for estimating the characteristics of SJS,
as they could hamper the search for optimal solutions, and suggest potential improvements.
\end{abstract}


\section{Introduction}

The notion of \emph{Sparse Joint Shift (SJS)} was introduced by Chen et al.~\cite{chen&zaharia&Zou:SJS} as
a tractable model of dataset shift ``which considers the joint shift of both labels and a few features''.
In this paper, we re-analyse the notion in some depth, looking closer at its connection to
prior probability shift and the link between SJS and covariate shift. For facilitating this intent,
we take recourse to some observations on general dataset shift
by Tasche~\cite{tasche2022factorizable}.

We revisit the notion of SJS in a probabilistic measure theory setting that allows for the
easy deployment of powerful mathematical
tools like the tower property of conditional expectation or Dynkin's theorem which are less readily available
within the density-focussed approach used by Chen et al.~\cite{chen&zaharia&Zou:SJS}. This paper can be
read as a commentary on Chen et al.~\cite{chen&zaharia&Zou:SJS} but also as a self-contained presentation of SJS.

The main contributions with this paper to the subject are as follows:
\begin{itemize}
\item We extend the definition of SJS such that it also covers (for instance) distributional
shift of $\Vert \mathbf{X}\Vert = \sqrt{\sum_{i=1}^d X_i^2}$ if $\mathbf{X} = (X_1, \dots, X_d)$
is an $\mathbb{R}^d$-valued feature vector (see Remark~\ref{rm:recon} below), without having to enlarge
the set of available features.
\item We show that SJS for a set of features $X_i, i\in I,$ implies SJS for all sets of features
that include this set of features (Corollary~\ref{co:all} below).
\item We derive for SJS a conditional version of the posterior correction formula of Saerens et
 al.~\cite{saerens2002adjusting}
and Elkan~\cite{Elkan01}, see Proposition~\ref{pr:corrH} below.
\item In Theorem~\ref{th:ident}, we present a result on the identifiability of SJS that can be interpreted
as a generalisation
of the confusion matrix approach to class prior estimation under prior probability shift
(Saerens et al.~\cite{saerens2002adjusting}).
\item We explore in depth the relationship between SJS and covariate shift (Theorem~\ref{th:SJSvsCovShift} below).
\item We comment on the approaches to the estimation of importance weights proposed by
Chen et al.~\cite{chen&zaharia&Zou:SJS}, identify certain issues that might
affect their performance (Section~\ref{se:how} below), and suggest improvements to avoid these issues.
\end{itemize}

In Section~\ref{se:related}, an overview of related work on the subject is given. Section~\ref{se:setting}
sets the scene for the paper by specifying the model considered and introducing the notion of
sparse joint shift (SJS). Properties of SJS are discussed in some detail in Section~\ref{se:analyses}.
In Section~\ref{se:how}, two approaches by Chen et al.~\cite{chen&zaharia&Zou:SJS}
to the estimation of densities under the assumption of SJS are presented and commented on.

See Appendix~\ref{se:notation} for a brief glossary of concepts and notation
from probabilistic measure theory (like indicator function $\mathbf{1}_S$,
$\sigma$-algebra, or conditional expectation)  which are frequently used in the following.
Short proofs of results are integrated in the text, longer proofs are
presented in Appendix~\ref{se:proofs}. Appendix~\ref{se:instant} provides the formulae relevant
for applications of SJS in the special case of all features being discrete.

\section{Related work}
\label{se:related}

In a wide sense, this paper is related to all work on dataset shift, transfer learning, domain adaptation,
and class prior estimation in multinomial classification settings, including focussed literature on the binary case.

More specifically, there is closely related work on dataset shift models that are based on invariance assumptions:
\begin{itemize}
\item Covariate shift, including density ratio estimation
(Shimodaira~\cite{shimodaira2000improving}, Sugiyama et al.~\cite{sugiyama2012density}, and the references therein).
For this type of shift,
the posterior class probabilities are assumed to be invariant between source and target distribution.
\item Prior probability shift, also called label shift or target shift
 (Storkey~\cite{storkey2009training}, Azizzadenesheli~\cite{Azizzadenesheli2022}, and the references therein).
  In this case, the class-conditional
 feature distributions are assumed to be invariant between source and target distribution.
\item Factorizable joint shift (He et al.~\cite{he2022domain}). According to the original definition of He et al.,
at first glance, no invariance assumption is involved. Tasche~\cite{tasche2022factorizable} showed that with
this type of shift, the ratios of the class-conditional feature densities remain invariant up to constant
factors between source and target distribution.
\item Covariate shift with posterior drift (Scott~\cite{Scott2019}).  This type of shift is characterised by the
weak invariance assumption that the posterior probabilities of the source distribution are statistically
sufficient for the full feature vector also under the target distribution (Tasche~\cite{Tasche2022}).
\end{itemize}
Sparse joint shift (SJS) is also characterised by an invariance assumption: According to
Section~3 of Chen et al.~\cite{chen&zaharia&Zou:SJS},
``Roughly speaking, SJS allows both labels and a few features to shift,
but assumes the remaining features' conditional distribution to stay the same.''

An assumption of absolute continuity of the target distribution with respect to the source distribution
appears to be fundamental for the study of dataset shift with invariance assumptions.
Consequently, there is a clear demarcation between
the approach chosen in this paper and a large thread of literature that begins with the seminal paper
by Ben-David et al.~\cite{Ben-David2007Representations} and investigates ways involving
so-called representations (transformations of the features in order to make their
source and target domains similar) to get around the absolute continuity assumption.
See Johansson et al.~\cite{pmlr-v89-johansson19a} for a discussion of limitations to the representation-based approach
to domain adaptation.

A third general approach to estimating dataset shift is based on mixture models and solving transport problems,
under the assumption that the shift is enacted by moving as little as possible (probability)
mass (Hofer~\cite{hofer2015adapting}). Such an approach may also involve the construction of representations
(Kirchmeyer et al.~\cite{kirchmeyer2021mapping}).

\section{Setting}
\label{se:setting}

In this paper, we use the probabilistic setting of Tasche~\cite{tasche2022factorizable}. See also Section~2.2 of
Tasche~\cite{tasche2022class} for a detailed reconciliation of mainstream machine learning
and probabilistic measure theory settings.

\begin{assumption}\label{as:setting}{\it
$(\Omega, \mathcal{M})$ is a measurable space. The \emph{source distribution}
$P$ and the \emph{target distribution} $Q$ are probability measures on $(\Omega, \mathcal{M})$.
For some positive integer $\ell \ge 2$, events $A_1, \ldots, A_\ell \in \mathcal{M}$ and
a sub-$\sigma$-algebra $\mathcal{H} \subset \mathcal{M}$ are given. The events $A_i$, $i = 1, \ldots, \ell$,
and $\mathcal{H}$ have the following properties:
\begin{itemize}
\item[(i)] $\bigcup_{i=1}^\ell A_i = \Omega$.
\item[(ii)] $A_i \cap A_j = \emptyset$,\ $i, j = 1, \ldots, \ell$, $i\neq j$.
\item[(iii)] $0 < P[A_i]$,\ $i = 1, \ldots, \ell$.
\item[(iv)] $0 < Q[A_i]$,\ $i = 1, \ldots, \ell$.
\item[(v)] $A_i \notin \mathcal{H}$,\ $i=1, \ldots, \ell$.
\end{itemize}}
\end{assumption}

$\mathcal{H}$ represents the features in a classification problem,
typically denoted $x$ or $\mathbf{x}$ in the machine learning literature. We suppose that
the set of features has been selected and fixed once for all such that in the following $\mathcal{H}$
is also fixed unless stated otherwise. Think of $\mathcal{H}$ as the data referred to by
a `long list' of explanatory variables.

The events $A_i$ can be understood as markers or labels of classes (typically denoted $y$ in machine learning)
or layers with different feature
distributions. Statistical classification means predicting the classes $A_i$ based on possibly
incomplete information as represented by $\mathcal{H}$.

Two measures $P$ and $Q$ are mentioned in Assumption~\ref{as:setting} in order to reflect
dataset shift and the resulting tasks of transfer learning and domain adaptation.
These tasks have in common that they both require inference on
the full target (test) distribution $Q$ from partially known (by observations)
$Q$ on $\mathcal{H}$ (the marginal distribution of the features) and completely known source (training) distribution
$P$.

Assumption~\ref{as:setting}, complemented by Assumption~\ref{as:cont} below, can be reconciled with the setting
of Chen et al.~\cite{chen&zaharia&Zou:SJS} in the following way:
\begin{itemize}
\item $P$ is the source distribution (`domain') $\mathbb{P}_s$ and $Q$ is the target distribution
$\mathbb{P}_t$ of Chen et al.~\cite{chen&zaharia&Zou:SJS}.
\item Chen et al.~assume that $\mathbb{P}_s$ and $\mathbb{P}_t$ have got densities
or mass functions $p_s = p_s(\mathbf{x}, y)$ and
$p_t = p_t(\mathbf{x}, y)$ with respect to an unspecified measure. In this paper, we make consequent use
of the source distribution $P$ as the common reference measure such that only a density
$\bar{h}$ of the target distribution
$Q$ with respect to the source distribution $P$ is needed for expressing the relationship of $Q$ and $P$.
\item The density $\bar{h}$ is also studied by Chen et al.~\cite{chen&zaharia&Zou:SJS}.
They call it \emph{weight function}
and denote is by $w(\mathbf{x}, y)$. In the literature, $\bar{h}$ is also called \emph{importance weight}
(see, e.g., Zhang et al.~\cite{Zhang:2013:TargetShift} and the references therein).
\item The aim of Chen et al.~\cite{chen&zaharia&Zou:SJS} in their paper
is to estimate the weight function $w(\mathbf{x}, y)$
in order to obtain full knowledge of the target distribution.
\item In this paper, we rather focus on determining the target posterior class probabilities
$Q[A_i\,|\,\mathcal{H}]$. When the target feature distribution is known, its combination with the
posterior probabilities (equivalently to having recourse to the weight function) also
provides full knowledge of the joint distribution of features and
class labels in the target domain. Then, in particular, the class prior probabilities $Q[A_i]$ can
be computed according to the law of total probability as
\begin{equation*}
Q[A_i]\ = \ E_Q\bigl[Q[A_i\,|\,\mathcal{H}]\bigr], \qquad i = 1, \ldots, \ell,
\end{equation*}
thus providing a solution to the class prior estimation problem (also called
`quantification') in the target domain.
\end{itemize}

The following two definitions introduce crucial notation for the subsequent discussion of
dataset shifts.
\begin{definition}
We denote by $\mathcal{A} = \sigma(\{A_1, \ldots, A_\ell\})$ the minimal sub-$\sigma$-algebra of
$\mathcal{M}$ containing all $A_i$,\ $i=1, \ldots, \ell$.
\end{definition}
Note that the $\sigma$-algebra $\mathcal{A}$ can be represented as
\begin{subequations}
\begin{equation}\label{eq:Asigma}
\mathcal{A} \ =\
\bigl\{\bigcup_{i=1}^\ell (A_i \cap M_i): M_1, \ldots, M_\ell \in \{\emptyset, \Omega\}\bigr\},
\end{equation}
while the minimal $\sigma$-algebra $\sigma(\mathcal{H} \cup \mathcal{A})$ that contains both $\mathcal{H}$ and
$\mathcal{A}$ can be written as
\begin{equation}
\sigma(\mathcal{H} \cup \mathcal{A})  \ =\
\bigl\{\bigcup_{i=1}^\ell (A_i \cap H_i): H_1, \ldots, H_\ell \in \mathcal{H}\bigr\}.\label{eq:Hbar}
\end{equation}
\end{subequations}

\begin{definition}\label{de:condDist}
Under Assumption~\ref{as:setting}, define the following class-conditional distributions,
by letting for $M \in \mathcal{M}$  and
$i = 1, \ldots, \ell$
\begin{equation}\label{eq:classcond}
P_i[M] = P[M\,|\,A_i]= \frac{P[A_i\cap M]}{P[A_i]}\quad \text{and}\quad Q_i[M] = Q[M\,|\,A_i]
= \frac{Q[A_i\cap M]}{Q[A_i]}.
\end{equation}
\end{definition}
In the literature, when restricted to the feature information set $\mathcal{H}$, the $P_i$ and $Q_i$
sometimes are called \emph{class-conditional feature distributions}.

The situation where $P[M] \neq Q[M]$ holds for at least one $M\in \sigma(\mathcal{H} \cup \mathcal{A})$
is called \emph{dataset shift}
(Definition~1 of Moreno-Torres et al.~\cite{MorenoTorres2012521})
and is also known as `distribution shift' or `population drift'.
Dataset shift is called ``joint distribution shift'' by
Chen et al.~\cite{chen&zaharia&Zou:SJS} in the case where the joint distribution of features and
labels is shifted in contrast
to shifts of one of the marginal distributions of the labels or features only.

The following special type of dataset shift, `sparse joint shift', was
introduced by Chen et al.~\cite{chen&zaharia&Zou:SJS}, in a slightly less general way.
\begin{definition}[Sparse Joint Shift]\label{de:SJS}
Under Assumption~\ref{as:setting}, assume that $\mathcal{F}$ is a sub-$\sigma$-algebra of
$\mathcal{H}$. Then the target distribution
$Q$ is related to the source distribution $P$ through $(\mathcal{F}, \mathcal{H})$-\emph{sparse joint shift} (SJS)
if for all $H \in \mathcal{H}$ and $i \in \{1, \ldots, \ell\}$ it holds that
\begin{equation}\label{eq:SJS}
Q_i[H\,|\,\mathcal{F}] \ = \ P_i[H\,|\,\mathcal{F}].
\end{equation}
\end{definition}
\eqref{eq:SJS} is understood in the sense that for each $H\in\mathcal{H}$ and each $i = 1, \ldots, \ell$
there is an $\mathcal{F}$-measurable random variable $Z$ such that for all $F\in \mathcal{F}$ it holds that
\begin{equation*}
P_i[F \cap H] = E_{P_i}[\mathbf{1}_F\,Z] \qquad\text{and}\qquad Q_i[F \cap H] = E_{Q_i}[\mathbf{1}_F\,Z].
\end{equation*}

\begin{remark}
Observe that if $Q$ is related to the source distribution $P$ through $(\mathcal{F}, \mathcal{H})$-SJS then
$Q$ also is related to the source distribution $P$ through $(\mathcal{F}, \mathcal{H}')$-SJS for
any sub-$\sigma$-algebra $\mathcal{H}'$ of $\mathcal{H}$ with $\mathcal{H}' \supset \mathcal{F}$.
In the following, we use the short hand $\mathcal{F}$-SJS for $(\mathcal{F}, \mathcal{H})$-SJS
whenever there is no risk of ambiguity.\hfill \qed
\end{remark}

\begin{remark}\label{rm:recon}
Definition~\ref{de:SJS} can be reconciled with the definition of sparse joint shift of Definition~1 of
Chen et al.~\cite{chen&zaharia&Zou:SJS} as follows:
\begin{itemize}
\item Consider a feature vector $\mathbf{X} = (X_1, \dots, X_d)$ mapping $\Omega$ into $\mathbb{R}^d$.
\item $\mathbf{X}$ corresponds to $\mathcal{H}$ in Assumption~\ref{as:setting} if $\mathcal{H}$
is defined as the smallest $\sigma$-algebra such that $\mathbf{X}$ is $\mathcal{H}$-measurable, i.e.\
$\mathcal{H}= \sigma(\mathbf{X})$.
\item If $I$ is a subset of $\{1, \ldots, d\}$ with at most $m$ elements, then choosing
$\mathcal{F} = \sigma(X_i,\, i \in I)$
in Definition~\ref{de:SJS} replicates the $m$-SJS definition of
Chen et al.~\cite{chen&zaharia&Zou:SJS} if Proposition~\ref{pr:properties}~(iii) below is taken into account with
\begin{equation*}
\mathcal{G}\ =\ \sigma(X_i,\, i \in \{1, \ldots, d\}\setminus I).
\end{equation*}
Chen et al.~describe
the set $(X_i,\, i \in I)$ as the set of \emph{shifted features}.
\end{itemize}
However, Definition~\ref{de:SJS} allows for shifts not only of distributions of subsets of features like in
the definition of $m$-SJS of Chen et al.~\cite{chen&zaharia&Zou:SJS}
but also for distribution shifts of functions of the features. For instance,
in Definition~\ref{de:SJS}, $\mathcal{F}$ may equal $\sigma(\Vert \mathbf{X}\Vert)$ where $\Vert \cdot \Vert$
denotes the usual Euclidean norm $\Vert \mathbf{X}\Vert = \sqrt{\sum_{i=1}^d X_i^2}$ of the feature vector.\hfill \qed
\end{remark}
\begin{subequations}
Note that \eqref{eq:SJS} is equivalent to
\begin{equation}\label{eq:eqSJS}
Q\bigl[H\,|\,\sigma(\mathcal{F}\cup\mathcal{A})\bigr]\ =\ P\bigl[H\,|\,\sigma(\mathcal{F}\cup\mathcal{A})\bigr]
\end{equation}
since thanks to the generalized Bayes formula (Theorem~10.8 of Klebaner~\cite{Klebaner}) we have
$\mu_i[H\,|\,\mathcal{F}] = \frac{\mu[H\cap A_i\,|\,\mathcal{F}]}{\mu[A_i\,|\,\mathcal{F}]}$ and therefore
\begin{equation}\label{eq:genBayes}
\sum_{i=1}^\ell \mathbf{1}_{A_i}\, \mu_i[H\,|\,\mathcal{F}] \ = \
\mu\bigl[H\,|\,\sigma(\mathcal{F}\cup\mathcal{A})\bigr],
\end{equation}
where $\mu$ is $P$ or $Q$.
\end{subequations}

Consider the following two special cases in Definition~\ref{de:SJS}:
\begin{itemize}
\item $\mathcal{F} = \mathcal{H}$. Then \eqref{eq:SJS} is equivalent to $\mathbf{1}_H = \mathbf{1}_H$ such
that under Assumption~\ref{as:setting} each target distribution $Q$ is related to the source
distribution $P$ through $\mathcal{H}$-SJS. Hence,
Definition~\ref{de:SJS} is not particularly meaningful in this case but it does no harm to include it
in the definition.
\item $\mathcal{F} = \{\emptyset, \Omega\}$. Then \eqref{eq:SJS} is equivalent to
\begin{equation*}
Q_i[H] \ = \ P_i[H], \quad \text{for}\ i=1, \ldots, \ell \ \text{and}\ H \in \mathcal{H}.
\end{equation*}
Hence $\{\emptyset, \Omega\}$-SJS is nothing else but prior probability shift in the sense of
Definition~3 of Moreno-Torres et al.~\cite{MorenoTorres2012521}.
\end{itemize}

If $Q$ and $P$ from Assumption~\ref{as:setting} are mutually singular, i.e.\ there is
an event $\Omega' \in \mathcal{M}$ such that $P[\Omega'] = 1$ and $Q[\Omega'] = 0$, then
\eqref{eq:SJS} holds true for any $\mathcal{F} \subset \mathcal{H}$ because conditional
probabilities are uniquely defined only up to events of probability $0$. Thus it makes
sense to exclude this case by an additional assumption of absolute continuity (as implicitly also did
Chen et al.~\cite{chen&zaharia&Zou:SJS}) which basically means that impossible events under the source distribution
are also impossible under the target distribution.
With this natural assumption,
important properties of SJS can be proved. See Appendix~\ref{se:notation} for
explanations of the notation involved.

\begin{assumption}\label{as:cont}{\it
Assumption~\ref{as:setting} holds, and $Q$ is absolutely continuous with respect to $P$ on
$\sigma(\mathcal{H} \cup \mathcal{A})$, i.e.,
\begin{equation*}
Q|\sigma(\mathcal{H} \cup \mathcal{A}) \ \ll\  P|\sigma(\mathcal{H} \cup \mathcal{A}),
\end{equation*}
where $\mu|\sigma(\mathcal{H} \cup \mathcal{A})$ stands for the measure $\mu \in\{P, Q\}$ with domain restricted to
$\sigma(\mathcal{H} \cup \mathcal{A})$.}
\end{assumption}

Why requesting $Q \ll P$ instead of $P \ll \mu$ and $Q\ll \mu$ as is common in the machine learning
literature, for instance with $\mu =$ Lebesgue measure? Doing so has a couple of advantages:
\begin{itemize}
\item The notation is simplified.
\item $Q \ll P$ is implied by $P \ll \mu$ and $Q\ll \mu$
if $\mu\left[\frac{d P}{d \mu} =0, \frac{d Q}{d \mu}>0\right]=0$.
This is a natural assumption, see the discussion in Tasche~\cite{tasche2022factorizable}.
\item Special mathematical tools are available, e.g.\ the Radon-Nikodym theorem or the conditional
expectation operator for determining densities on sub-$\sigma$-algebras.
\end{itemize}

In practical applications, under Assumption~\ref{as:cont}, the density $\bar{h}=\frac{d Q}{d P}$
must be estimated as density
ratio $\frac{d Q}{d \mu} / \frac{d P}{d \mu}$. For literature on this, see
Sugiyama et al.~\cite{sugiyama2012density} and
references therein.

The following proposition lists some useful elementary properties of SJS under the absolute continuity assumption.
\begin{subequations}
\begin{proposition}\label{pr:properties}
Under Assumption~\ref{as:cont}, suppose that $\mathcal{F}$ is a sub-$\sigma$-algebra of
$\mathcal{H}$. Then the following three statements are equivalent:
\begin{itemize}
\item[(i)] $Q$ is related to $P$ through $\mathcal{F}$-SJS in the sense of Definition~\ref{de:SJS}.
\item[(ii)] For each
$\mathcal{H}$-measurable non-negative random variable $X$ it holds that
\begin{equation}\label{eq:X}
E_{Q_i}[X\,|\,\mathcal{F}] \ = \ E_{P_i}[X\,|\,\mathcal{F}], \qquad i = 1, \ldots, \ell.
\end{equation}
\item[(iii)] There is a sub-$\sigma$-algebra $\mathcal{G}$ of
$\mathcal{H}$ such that $\mathcal{H} = \sigma(\mathcal{F}\cup\mathcal{G})$ and
\begin{equation}\label{eq:restrictSJS}
Q_i[G\,|\,\mathcal{F}] \ = \ P_i[G\,|\,\mathcal{F}], \qquad\text{for all}\
G\in \mathcal{G}\ \text{and}\ i = 1, \ldots, \ell.
\end{equation}
\end{itemize}
\end{proposition}
\end{subequations}

See Appendix~\ref{se:proofs} for a proof of Proposition~\ref{pr:properties}.

\section{Analyses of sparse joint shift}
\label{se:analyses}

In the following, we assume that the target distribution $Q$ is absolutely continuous with respect
to $P$. By allowing us to take recourse to results from Tasche~\cite{tasche2022factorizable}, this assumption
facilitates thorough analyses of SJS.

\subsection{Characterising SJS}

By Lemma~1 of Tasche~\cite{tasche2022factorizable}, under Assumption~\ref{as:cont} the target
conditional feature distributions
$Q_i|\mathcal{H}$, $i=1, \ldots,\ell$, are absolutely continuous with respect to the source conditional feature
distributions $P_i|\mathcal{H}$, $i=1, \ldots,\ell$. For the respective densities, we write
\begin{subequations}
\begin{equation}\label{eq:hi}
h_i \ =\ \frac{d Q_i|\mathcal{H}}{d P_i|\mathcal{H}}, \qquad    i = 1, \ldots, \ell.
\end{equation}

By Theorem~1 of Tasche~\cite{tasche2022factorizable}, under Assumption~\ref{as:cont} the density
$\bar{h}$ of $Q$ with
respect to $P$ on the sub-$\sigma$-algebra $\sigma(\mathcal{H} \cup \mathcal{A})$ of $\mathcal{M}$ can be represented as
\begin{equation}\label{eq:form}
\bar{h} \ = \ \sum_{i=1}^\ell h_i\,\frac{Q[A_i]}{P[A_i]}\,\mathbf{1}_{A_i}.
\end{equation}
\end{subequations}

\begin{theorem}[Equivalent conditions for SJS]\label{th:eqSJS}
Under Assumption~\ref{as:cont}, let $\mathcal{F}$  be a sub-$\sigma$-algebra of
$\mathcal{H}$. Then the following three
statements are equivalent:
\begin{itemize}
\item[(i)] The target distribution
$Q$ is related to the source distribution $P$ through $\mathcal{F}$-SJS.
\item[(ii)] For each $i=1,\ldots, \ell$, there exists an $\mathcal{F}$-measurable density $f_i$ of
the target conditional feature distribution $Q_i$ with respect to the source conditional feature
distribution $P_i$ on $\mathcal{H}$.
\item[(iii)] There exists a $\sigma(\mathcal{F}\cup\mathcal{A})$-measurable density $\bar{f}$ of
$Q$ with respect to $P$ on  $\sigma(\mathcal{H}\cup\mathcal{A})$.
\end{itemize}
Moreover, if $Q$ is related to the source distribution $P$ through $\mathcal{F}$-SJS then
the density $\bar{f}$ mentioned in (iii) can be represented as
\begin{equation}\label{eq:SJSdens}
\bar{f} \ = \ \sum_{i=1}^\ell f_i\,\frac{Q[A_i]}{P[A_i]}\,\mathbf{1}_{A_i},
\end{equation}
with $f_i$, $i = 1, \ldots, \ell$, as in (ii).
\end{theorem}

Theorem~\ref{th:eqSJS} is a rather straightforward consequence of the following observation.

\begin{lemma}\label{le:measurable}
Let $Q$ and $P$ be probability measures on the measurable space $(\Omega, \mathcal{M})$. Suppose that
$Q$ is absolutely continuous with respect to $P$ on $\mathcal{M}$ and that
$\mathcal{G}$ is a sub-$\sigma$-algebra of $\mathcal{M}$. Then the two following statements are
equivalent:
\begin{itemize}
\item[(i)] For all $M\in \mathcal{M}$, it holds that $P[M\,|\,\mathcal{G}] = Q[M\,|\,\mathcal{G}]$.
\item[(ii)] There exists a $\mathcal{G}$-measurable density $g$ of $Q$ with respect to $P$ on $\mathcal{M}$.
\end{itemize}
If $\varphi$ is any $\mathcal{M}$-measurable density of $Q$ with respect to $P$ on $\mathcal{M}$ and
(i) or (ii) are true, then $g=E_P[\varphi\,|\,\mathcal{G}]$ is a density of $Q$ with respect to $P$ on $\mathcal{M}$.
\end{lemma}
\begin{proof}[\textbf{Proof of Lemma~\ref{le:measurable}}]
Assume that $P[M\,|\,\mathcal{G}] = Q[M\,|\,\mathcal{G}]$ is true for all $M\in \mathcal{M}$. Choose any
density $\varphi$ of $Q$ with respect to $P$ on $\mathcal{M}$ and define $g=E_P[\varphi\,|\,\mathcal{G}]$.
For any $M \in \mathcal{M}$ then we obtain
\begin{multline*}
Q[M] = E_Q\bigl[Q[M\,|\,\mathcal{G}]\bigr] = E_P\bigl[\varphi\,Q[M\,|\,\mathcal{G}]\bigr]
= \\ E_P\bigl[\varphi\,P[M\,|\,\mathcal{G}]\bigr] = E_P\bigl[g\,P[M\,|\,\mathcal{G}]\bigr] =
E_P[g\,\mathbf{1}_M].
\end{multline*}
This proves (i) $\Rightarrow$ (ii) and the statement on how to choose $g$.

Assume now that (ii) is true. Fix $M\in \mathcal{M}$. For any $G\in\mathcal{G}$ then we obtain
\begin{equation*}
E_Q\bigl[\mathbf{1}_G\,P[M\,|\,\mathcal{G}]\bigr] = E_P\bigl[g\,\mathbf{1}_G\,P[M\,|\,\mathcal{G}]\bigr] =
E_P[g\,\mathbf{1}_{G\cap M}] = Q[G\cap M].
\end{equation*}
By the definition of conditional probability, this implies (ii) $\Rightarrow$ (i).
\end{proof}

\begin{proof}[\textbf{Proof of Theorem~\ref{th:eqSJS}}]
Show (i) $\Rightarrow$ (ii): This follows by applying Lemma~\ref{le:measurable} to each of the
pairs  $Q_i$ and $P_i$ from the definition of $\mathcal{F}$-SJS.

Show (ii) $\Rightarrow$ (iii): For each $i=1,\ldots, \ell$, choose an $\mathcal{F}$-measurable density $f_i$
of $Q_i$ with respect to $P_i$ on $\mathcal{H}$. Then, by representation \eqref{eq:form}, $\bar{f}$ defined
by \eqref{eq:SJSdens}
is a $\sigma(\mathcal{F}\cup\mathcal{A})$-measurable density of $Q$ with respect to $P$ on $\sigma(\mathcal{H}
    \cup\mathcal{A})$.

Show (iii) $\Rightarrow$ (i): Follows from the equivalence of \eqref{eq:SJS} for all $i = 1, \ldots, \ell$ and
\eqref{eq:eqSJS}, combined with Lemma~\ref{le:measurable}.
\end{proof}

Note that SJS does not imply in general that there exists an $\mathcal{F}$-measurable density $f$ of
$Q$ with respect to $P$ on $\mathcal{H}$. We show in Section~\ref{se:SJSvsCov} below that subject to a mild
condition
the existence of such a density $f$ is implied by the joint presence of SJS and covariate shift.
Moreover, if SJS holds and there is an $\mathcal{F}$-measurable density $f$ of
$Q$ with respect to $P$ on $\mathcal{H}$ then $Q$ and $P$ are also related through covariate
shift, again under a mild condition.

As formally stated in the following corollary, Theorem~\ref{th:eqSJS} implies that SJS is a nested property
in the sense that SJS for a set of features  implies SJS for all sets of features
that include the first set. Choose $\mathcal{F}=\sigma(X_i,\, i\in I)$ and $\mathcal{F}'=\sigma(X_i,\, i\in I')$
for features $X_i$ and index sets $I \subset I'$ to see this.
In particular, prior probability shift implies SJS for any set of features.

\begin{corollary}\label{co:all}
Under Assumption~\ref{as:cont}, suppose that $\mathcal{F}$ is a sub-$\sigma$-algebra of
$\mathcal{H}$.
 If $Q$ is related to $P$ through $\mathcal{F}$-SJS and $\mathcal{F}' \supset
    \mathcal{F}$ is another sub-$\sigma$-algebra of $\mathcal{H}$ then
    $Q$ also is related to $P$ through $\mathcal{F}'$-SJS.
\end{corollary}

\begin{proof}[\textbf{Proof of Corollary~\ref{co:all}}]
By Theorem~\ref{th:eqSJS}~(ii),
for each $i=1,\ldots, \ell$, there exists an $\mathcal{F}$-measurable density $f_i$ of
$Q_i$ with respect to $P_i$ on $\mathcal{H}$. For $\mathcal{F}' \supset
    \mathcal{F}$, each $f_i$ is $\mathcal{F}'$-measurable, too. Hence, by the implication
    `(ii) $\Rightarrow$ (i)' of Theorem~\ref{th:eqSJS}, $Q$ is related to $P$
    through  $\mathcal{F}'$-SJS.
\end{proof}
\begin{remark}\label{rm:LeaveOneOut}
Consider the setting of the paper of Chen et al.~\cite{chen&zaharia&Zou:SJS} as summarised in
Remark~\ref{rm:recon} above.
If $\mathcal{F} = \sigma(X_i, i\in I)$ for some $I \subsetneq \{1, \ldots, d\}$, then there is some
$j \in \{1, \ldots, d\}$ such $I \subset J = \{1, \ldots, d\}\setminus \{j\}$. It follows that
\begin{equation}
\mathcal{F}' \ = \ \sigma(X_i, i \in J)\ \supset \ \mathcal{F}.
\end{equation}
By Corollary~\ref{co:all}, therefore for fitting any $m$-SJS with $m < d$ in the sense of
Chen et al.~\cite{chen&zaharia&Zou:SJS}, it
suffices to consider index sets with $d-1$ elements, i.e.\ $(d-1)$-SJS.\hfill \qed
\end{remark}

The following Eq.~\eqref{eq:corrH} could be called \emph{conditional posterior correction formula}. On the one hand,
it generalises the scope of
the correction formula presented by Eq.~(2.4) of Saerens et al.~\cite{saerens2002adjusting} and Theorem~2 of
Elkan~\cite{Elkan01}
from prior probability shift to SJS. On the other hand, it represents a special case of the correction formula
proven in Corollary~2 of Tasche~\cite{tasche2022factorizable}.

\begin{subequations}
\begin{proposition}\label{pr:corrH}
Under Assumption~\ref{as:cont}, let $Q$ be related to $P$ through $\mathcal{F}$-SJS.
Then it $Q$-almost surely holds that
\begin{align}\label{eq:greater0}
0 & \ <\ \sum_{j=1}^\ell \frac{Q[A_j\,|\,\mathcal{F}]}{P[A_j\,|\,\mathcal{F}]} P[A_j\,|\,\mathcal{H}],
\quad \text{and}\\
Q[A_i\,|\,\mathcal{H}] & \ = \ \frac{\frac{Q[A_i\,|\,\mathcal{F}]}{P[A_i\,|\,\mathcal{F}]} P[A_i\,|\,\mathcal{H}]}
{\sum_{j=1}^\ell \frac{Q[A_j\,|\,\mathcal{F}]}{P[A_j\,|\,\mathcal{F}]} P[A_j\,|\,\mathcal{H}]}, \quad
    i = 1\,\ldots, \ell.\label{eq:corrH}
\end{align}
In particular, the right-hand side
of \eqref{eq:corrH} is $Q$-almost surely well-defined under the convention that $0/0 = 0$.
\end{proposition}
\end{subequations}

See Appendix~\ref{se:proofs} for a proof of Proposition~\ref{pr:corrH}.

\subsection{Identifiability in the presence of sparse joint shift}

Identifiability of the target distribution $Q$ under SJS is described in Section~3.1 of
Chen et al.~\cite{chen&zaharia&Zou:SJS}
as follows: ``If a joint feature and label distribution matches
the target feature distribution and satisfies the $m$-SJS requirement together with the source distribution,
it has to be the target distribution.'' Chen et al.~\cite{chen&zaharia&Zou:SJS} present in their Theorem~1 a
sufficient
condition for identifiability phrased in terms of linear independence of components of the density of the joint
source distribution. An alternative to their result is provided below in Theorem~\ref{th:ident}. It is based
on the following observation.

\begin{lemma} \label{le:total}
Let $(\Omega, \mathcal{M}, \mu)$ a probability space, $\mathcal{G}$ a sub-$\sigma$-algebra of $\mathcal{M}$, and
$A_1, \ldots, A_\ell \in \mathcal{M}$ a decomposition of $\Omega$, i.e.\ properties (i) and (ii) of
Assumption~\ref{as:setting} hold. In addition, assume that $\mu(A_i) > 0$ for all $i = 1, \ldots, \ell$, and define
\begin{equation*}
\mu_i(M) \ = \ \frac{\mu(M \cap A_i)}{\mu(A_i)}, \qquad M\in\mathcal{M}, i = 1, \ldots, \ell.
\end{equation*}
Let $X_1, \ldots, X_n$ be $\mathcal{M}$-measurable non-negative random variables and define the random matrix
$R = (R_{ij})_{1 \le i \le n, 1 \le j \le \ell} = R(\mu, \mathcal{G}, n)$ by
\begin{subequations}
\begin{equation}\label{eq:defR}
R_{i j} \ = \ E_{\mu_j}[X_i\,|\,\mathcal{G}].
\end{equation}
Then it follows that
\begin{equation}\label{eq:matrix}
\begin{pmatrix}
E_\mu[X_1\,|\,\mathcal{G}] \\
\vdots \\
E_\mu[X_n\,|\,\mathcal{G}]
\end{pmatrix}
\ = \ R \times
\begin{pmatrix}
\mu[A_1\,|\,\mathcal{G}] \\
\vdots \\
\mu[A_\ell\,|\,\mathcal{G}]
\end{pmatrix}.
\end{equation}
\end{subequations}
\end{lemma}

See Appendix~\ref{se:proofs} for a proof of Lemma~\ref{le:total}. On the one hand, the following Theorem~\ref{th:ident}
provides a sufficient condition for identifiability under SJS. On the other hand, on the basis of
Eq.~\eqref{eq:obtain} from its proof, it can be interpreted
as a conditional version of the `confusion matrix' method (Section~2.3.1 of
Saerens et al.~\cite{saerens2002adjusting};
also called `Adjusted Count', Forman~\cite{forman2008quantifying}) for estimating the class prior probabilities
under prior probability shift.

\begin{theorem}[Identifiability]\label{th:ident}
Under Assumption~\ref{as:setting}, let $\mathcal{F}$ and $\mathcal{F}'$ be sub-$\sigma$-algebras of
$\mathcal{H}$ and $Q$ and $Q'$ be probability measures on $(\Omega, \sigma(\mathcal{H}\cup\mathcal{A}))$ such
that $Q|\mathcal{H} = Q'|\mathcal{H}$ and both $Q$ and $Q'$ are absolutely continuous with respect to $P$ on
$\sigma(\mathcal{H}\cup\mathcal{A})$.
Suppose that  $Q$ is related to $P$ through $\mathcal{F}$-SJS and $Q'$ is related to $P$ through
$\mathcal{F}'$-SJS.
Assume that $X_1, \ldots, X_\ell$ are $\mathcal{H}$-measurable non-negative random variables
and
define the random matrix $R=R(P, \sigma(\mathcal{F}\cup\mathcal{F}'), \ell)$ by \eqref{eq:defR}, i.e.\
with $\mu = P$, $\mathcal{G} = \sigma(\mathcal{F}\cup\mathcal{F}')$ and $n = \ell$.
If the matrix $R$ satisfies the condition
\begin{equation}\label{eq:rank}
P[\mathrm{rank}(R) = \ell] \ = \ 1,
\end{equation}
then $Q[M] = Q'[M]$ follows for all $M \in \sigma(\mathcal{H}\cup\mathcal{A})$.
\end{theorem}

See Appendix~\ref{se:proofs} for a proof of Theorem~\ref{th:ident}.
The sufficient condition for identifiability presented in Theorem~1 Chen et al.~\cite{chen&zaharia&Zou:SJS} implicitly
involves picking a set of $\ell$ feature values in order to obtain a system of $\ell$ equations
for $\ell$ unknowns with a unique solution, for each point in the domain of the weight function linking
source and target distributions. In Theorem~\ref{th:ident}, a corresponding but more transparent selection of
random variables $X_i$ is introduced, and `good'
choices like an accurate classifier or posterior probabilities suggest themselves.

How strong is Theorem~\ref{th:ident}? Obviously, its strength depends on the answer to the question
whether we can find $\mathcal{H}$-measurable non-negative random variables $X_1, \ldots, X_\ell$ such
that \eqref{eq:rank} holds, and if so, how easy it is.

Denote by $\mathbf{C} = (C_1, \ldots, C_\ell)$ hard (or crisp) multinomial classifiers in the sense that
\begin{equation}\label{eq:classifier}
\begin{split}
 & C_i \in \mathcal{H}\ \text{for all}\ i = 1, \ldots, \ell,  \\
& C_1, \ldots, C_\ell\ \text{is a disjoint decomposition of}\ \Omega,\ \text{and}  \\
& A_i\ \text{is predicted when}\ C_i\ \text{is observed.}
\end{split}
\end{equation}
Intuitively, $X_i = \mathbf{1}_{C_i}$
for some accurate classifier $\mathbf{C}$ might be a good choice for Theorem~\ref{th:ident}. Then $\mathbf{1}_{C_i}
\approx \mathbf{1}_{A_i}$, $i = 1, \ldots, \ell$, should hold with the consequence that also
\begin{equation*}
P_j[C_i\,|\,\mathcal{G}] \ \approx\ \begin{cases}
1, & \text{for}\ i = j,\\
0, & \text{for}\ i \neq j,
    \end{cases}
\end{equation*}
and with the further consequence that the matrix $R$ would be likely to have maximal rank.

A similar choice is $X_i = P[A_i\,|\,\mathcal{H}]$, $i = 1, \ldots, \ell$. With this choice, we can show
that in the binary case $\ell =2$, condition \eqref{eq:rank} has an intuitive  interpretation.
\begin{subequations}
\begin{lemma}\label{le:binary}
In the setting of Theorem~\ref{th:ident}, consider the case $\ell =2$ and let
\begin{equation}\label{eq:defX}
X_1 \ = P[A_1\,|\,\mathcal{H}]\quad \text{and}\quad
X_2 \ = P[A_2\,|\,\mathcal{H}].
\end{equation}
Then condition \eqref{eq:rank} is satisfied if and only
\begin{equation}\label{eq:binary}
P\Big[P[A_1\,|\,\mathcal{H}] = P\bigl[A_1\,|\,\sigma(\mathcal{F}\cup\mathcal{F}')\bigr]\Big] \ < \ 1.
\end{equation}
\end{lemma}
\end{subequations}

See Appendix~\ref{se:proofs} for a proof of Lemma~\ref{le:binary}.
Condition~\eqref{eq:binary}, in particular,  implies in the context of Lemma~\ref{le:binary} and
Theorem~\ref{th:ident} that the $\sigma$-algebra (or information set) $\sigma(\mathcal{F}\cup\mathcal{F}')$
must not be sufficient for $\mathcal{H}$ in the sense of Definition~3.1 of Tasche~\cite{Tasche2022}.
The definition of sufficiency is formally given below as we are having a closer look at the relationship between
sufficiency and SJS in the remainder of this section.
Recall that by assumption $\mathcal{H}$ reflects
the totality of observable information on the features available for the prediction of the labels.
In non-technical terms, hence sufficiency means that all the information needed for predicting the
class labels is already captured with an information set smaller than $\mathcal{H}$.
\begin{definition}[Statistical sufficiency]\label{de:sufficiency}
Let $(\Omega, \mathcal{M}, P)$ be a probability space and $\mathcal{A}$, $\mathcal{F}$ and $\mathcal{H}$
be sub-$\sigma$-fields of $\mathcal{M}$ such that $\mathcal{F} \subset \mathcal{H}$.
Then $\mathcal{F}$ is \emph{sufficient} for $\mathcal{H}$
with respect to $\mathcal{A}$ if for all $A\in \mathcal{A}$
\begin{equation}\label{eq:s.general}
 P[A\,|\,\mathcal{H}]\ = \ P[A\,|\,\mathcal{F}].
\end{equation}
\end{definition}

Sufficiency
is of interest and desirable for dimension reduction
primarily in the case where (in the context of Lemma~\ref{le:binary})
\begin{equation}\label{eq:subsetneq}
\sigma(\mathcal{F}\cup\mathcal{F}')\ \subsetneq\ \mathcal{H}.
\end{equation}
Observe, however, that requiring \eqref{eq:subsetneq} for Lemma~\ref{le:binary} instead of \eqref{eq:binary}
would not work  because
\eqref{eq:binary} may be violated despite \eqref{eq:subsetneq} being true.

A variation on the theme of Lemma~\ref{le:binary} for general $\ell \ge 2$
immediately follows from Proposition~3.7 of
Tasche~\cite{Tasche2022}. It confirms in the general case that  sufficiency does not help to
achieve identifiability under SJS. Recall from \eqref{eq:genBayes} that
\begin{equation*}
\sum_{i=1}^\ell \mathbf{1}_{A_i}\, P_i[H\,|\,\mathcal{F}] \ = \
P\bigl[H\,|\,\sigma\bigl(\mathcal{F}\cup\mathcal{A}\bigr)\bigr],\qquad \text{for all}\ H\in\mathcal{H}.
\end{equation*}

\begin{proposition}\label{pr:sufficiency}
Let $(\Omega, \mathcal{M}, P)$ be a probability space. Let $A_1, \ldots, A_\ell\in \mathcal{M}$ be
a decomposition of $\Omega$ such that properties (i), (ii) and (iii) of Assumption~\ref{as:setting} hold.
Suppose that $\mathcal{F}$ and $\mathcal{H}$ are
sub-$\sigma$-algebras of $\mathcal{M}$, with $\mathcal{F} \subset \mathcal{H}$.
Then the following three statements are equivalent:
\begin{itemize}
\item[(i)] For $i = 1, \ldots, \ell$, it holds that  $P[A_i\,|\,\mathcal{H}] = P[A_i\,|\,\mathcal{F}]$, i.e.\
$\mathcal{F}$ is sufficient for $\mathcal{H}$ with respect to $\mathcal{A} = \sigma(A_i:i = 1, \ldots, \ell)$
in the sense of Definition~\ref{de:sufficiency}.
\item[(ii)] For all $H\in\mathcal{H}$, it holds that
$P\bigl[H\,|\,\sigma\bigl(\mathcal{F}\cup\mathcal{A}\bigr)\bigr] = P[H\,|\,\mathcal{F}]$.
\item[(iii)] With $P_i$ defined as in \eqref{eq:classcond}, it holds that
$E_{P_i}[X\,|\,\mathcal{F}] = E_P[X\,|\,\mathcal{F}]$ for all $\mathcal{H}$-measurable
non-negative random variables $X$ and all $i=1, \ldots, \ell$.
\end{itemize}
\end{proposition}

By Proposition~\ref{pr:sufficiency}, if in the setting of Theorem~\ref{th:ident} the sub-$\sigma$-algebra
$\mathcal{G} = \sigma(\mathcal{F}\cup\mathcal{F}')$ is sufficient for $\mathcal{H}$,
the rank of the random matrix $R$ is at most $1$
with probability $1$. Hence, the $\ell \times \ell$-matrix $R$ is not invertible such that Theorem~\ref{th:ident}
cannot be applied for proving uniqueness of $Q$ related to $P$ through $\mathcal{F}$-SJS.
Intuitively, this is not a surprise as Proposition~\ref{pr:sufficiency}~(i) may be interpreted as `$\mathcal{H}$ does
not provide any information about $A_i$ that could not be obtained through a sufficient $\mathcal{F}$'.

However, the situation is not as bad as it might appear because of this
observation. For Proposition~\ref{pr:CDI} below shows that under SJS, sufficiency is passed on
from the source distribution to the target distribution. Before looking at that, another useful
invariance property is introduced in the following definition.

\begin{definition}[Conditional distribution invariance]\label{de:CDI}
Let $P$ and $Q$ be probability measures on a measurable space $(\Omega, \mathcal{M})$. Suppose that
$\mathcal{F}$ is a sub-$\sigma$-algebra of $\mathcal{M}$. Then $P$ and $Q$ are related through
$\mathcal{F}$-\emph{conditional distribution invariance} on $\mathcal{M}$  if
\begin{equation}\label{eq:CDI}
P[M\,|\,\mathcal{F}] \ = \ Q[M\,|\,\mathcal{F}], \qquad \text{for all}\ M\in\mathcal{M}.
\end{equation}
\end{definition}

\begin{proposition}\label{pr:CDI}
Under Assumption~\ref{as:cont}, let $\mathcal{F}$ be a sub-$\sigma$-algebra of $\mathcal{H}$.
Suppose that $\mathcal{F}$ is sufficient for $\mathcal{H}$ with respect to $\mathcal{A}$
under $P$ in the sense of Definition~\ref{de:sufficiency} and that $Q$ is related to $P$
through $\mathcal{F}$-SJS.
Then $P$ and $Q$ are related through $\mathcal{F}$-conditional distribution invariance
on $\mathcal{H}$ in the sense of Definition~\ref{de:CDI}, and
$\mathcal{F}$ is also sufficient for $\mathcal{H}$ with respect to $\mathcal{A}$
under $Q$.
\end{proposition}

\begin{proof}[\textbf{Proof of Proposition~\ref{pr:CDI}}]
Proposition~\ref{pr:sufficiency}~(iii) and $\mathcal{F}$-SJS together imply that
$Q_i[H\,|\,\mathcal{F}] = P[H\,|\,\mathcal{F}]$
for $H\in \mathcal{H}$ and $i = 1, \ldots, \ell$. Therefore we obtain for all $H\in \mathcal{H}$
\begin{equation*}
Q\bigl[H\,|\,\sigma(\mathcal{F}\cup\mathcal{A})\bigr]\ = \ \sum\nolimits_{i=1}^\ell \mathbf{1}_{A_i}\,
Q_i[H\,|\,\mathcal{F}] \ = \ P[H\,|\,\mathcal{F}].
\end{equation*}
From this, it follows that
\begin{equation*}
Q[H\,|\,\mathcal{F}] = E_Q\bigl[Q[H\,|\,\sigma(\mathcal{F}\cup\mathcal{A})\,\big|\,\mathcal{F}]\bigr]
= E_Q\bigl[P[H\,|\,\mathcal{F}]\,\big|\,\mathcal{F}\bigr] = P[H\,|\,\mathcal{F}].
\end{equation*}
Hence, \eqref{eq:CDI} is true for $H \in \mathcal{H}$. Sufficiency under $Q$ follows from
Proposition~\ref{pr:sufficiency}
`(ii) $\Rightarrow$ (i)'.
\end{proof}

Property~\eqref{eq:CDI}, under Assumption~\ref{as:cont}, is observable and verifiable, for instance by examining
if there is an $\mathcal{F}$-measurable density of $Q$ with respect to $P$ on $\mathcal{H}$ (as a consequence of
Lemma~\ref{le:measurable}). Hence by Proposition~\ref{pr:CDI}, the presence of conditional distribution invariance
could suggest that $Q$ and $P$ are also related through $\mathcal{F}$-SJS. Alas, one has to be very careful
about this because Example~\ref{ex:CDInotSJS} below shows that conditional distribution invariance does not
imply SJS in general.

\subsection{SJS vs.\ covariate shift}
\label{se:SJSvsCov}

In this section, we show that the relationship between SJS and covariate shift is intertwined with the
question of whether or not the target
marginal distribution of the features $Q|\mathcal{H}$ has got an $\mathcal{F}$-measurable
density with respect to the source unconditional feature distribution $P|\mathcal{H}$, or in other words
if $P$ and $Q$ are related through $\mathcal{F}$-conditional distribution invariance.

\begin{definition}[Covariate shift]\label{de:covshift}
Under Assumption~\ref{as:setting}, let $\mathcal{F}$ be a sub-$\sigma$-algebra of $\mathcal{H}$. Then
the source distribution $P$ and the target distribution $Q$ are related through $\mathcal{F}$-\emph{covariate shift}
if it holds that
\begin{equation}\label{eq:covshift}
Q[A_i\,|\,\mathcal{F}] \ = \ P[A_i\,|\,\mathcal{F}], \qquad \text{for all}\ i = 1, \ldots, \ell.
\end{equation}
\end{definition}
The most relevant case in Definition~\ref{de:covshift} is incurred with $\mathcal{F} = \mathcal{H}$. However,
as we will see below, in the context of SJS also the case of $\mathcal{F}$-covariate shift with a proper
sub-$\sigma$-algebra $\mathcal{F}$ of $\mathcal{H}$ is of interest.

Chen et al.~\cite{chen&zaharia&Zou:SJS} showed in Theorem~2 by example that in general
SJS does neither imply prior probability shift nor
covariate shift. In the following example -- which is a slight modification of the example of
Chen et al. -- we show that SJS and covariate shift can be present at the same time in
proper distribution shift, i.e.\ in shift where source and target distribution truly differ.
Moreover, in Example~\ref{ex:notSJS} below we point out that covariate shift does not
imply SJS. These three observations together motivate
the results presented in this section, regarding the triangle-relationship of covariate shift, SJS and
conditional distribution invariance as introduced in Definition~\ref{de:CDI}.

\begin{example}[Covariate shift and SJS are not mutually exclusive]
We consider $\{0, 1\}$-valued variables $Y$, $X_1$ and $X_2$ under different source and target distributions
$P$ and $Q$ respectively:
\begin{itemize}
\item $P[Y=1] = 1/2 = P[Y=0]$, $Q[Y=1] = 0.6 = 1 - Q[Y=0]$.
\item $X_1$ and $X_2$ are independent conditional on $Y$, both under $P$ and $Q$.
\item The distributions of $X_1$ conditional on $Y$ are
specified by
\begin{gather*}
P[X_1=1\,|\,Y=1] = 0.6, \ P[X_1=1\,|\,Y=0] = 0.4\\
\text{and}\ Q[X_1=1\,|\,Y=1] = 1/2, \ Q[X_1=1\,|\,Y=0] = 1/2.
\end{gather*}
The example differs only in the conditional distributions of $X_1$ under $P$ from the example presented
Theorem~2 of Chen et al.~\cite{chen&zaharia&Zou:SJS}.
\item The distributions of $X_2$ conditional on $Y$ are
specified by
\begin{gather*}
P[X_2=1\,|\,Y=1] = 0.2, \ P[X_2=1\,|\,Y=0] = 0.6\\
\text{and}\ Q[X_2=1\,|\,Y=1] = 0.2, \ Q[X_2=1\,|\,Y=0] = 0.6.
\end{gather*}
\end{itemize}
To reconcile these choices with the setting of this paper as specified by Assumption~\ref{as:setting}, we
define $A_1 = \{Y=1\}$, $A_2 = \{Y=0\}$, $\mathcal{H} = \sigma(X_1, X_2)$ and $\mathcal{F}= \sigma(X_1)$.
Then it follows that
\begin{itemize}
\item $Q$ is related to $P$ through $\mathcal{F}$-SJS but not through $\{\emptyset, \Omega\}$-SJS (i.e.\
prior probability shift), and
\item $Q[A_1\,|\,\mathcal{F}] = P[A_1\,|\,\mathcal{F}]$ (i.e.\ $\mathcal{F}$-covariate shift) because
\begin{align*}
P[A_1\,|\,\mathcal{F}] & = P[Y=1\,|\,X_1=1]\,\mathbf{1}_{\{X_1=1\}} + P[Y=1\,|\,X_1=0]\,\mathbf{1}_{\{X_1=0\}}\\
& = \frac{P[Y=1]\,P[X_1=1\,|\,Y=1]\,\mathbf{1}_{\{X_1=1\}}}{P[Y=1]\,P[X_1=1\,|\,Y=1] + P[Y=0]\,P[X_1=1\,|\,Y=0]}\, +\\
& \qquad  \ \frac{P[Y=1]\,P[X_1=0\,|\,Y=1]\,\mathbf{1}_{\{X_1=0\}}}
    {P[Y=1]\,P[X_1=0\,|\,Y=1] + P[Y=0]\,P[X_1=0\,|\,Y=0]}\\
& = 0.6\,\mathbf{1}_{\{X_1=1\}} + 0.4\,\mathbf{1}_{\{X_1=0\}} \\
& = Q[Y=1\,|\,X_1=1]\,\mathbf{1}_{\{X_1=1\}} + Q[Y=1\,|\,X_1=0]\,\mathbf{1}_{\{X_1=0\}}\\
& = Q[A_1\,|\,\mathcal{F}].
\end{align*}
\end{itemize}
Proposition~\ref{pr:corrH} now implies that $Q[A_1\,|\,\mathcal{H}] = P[A_1\,|\,\mathcal{H}]$,
i.e.\ $Q$ and $P$ are related through (full) $\mathcal{H}$-covariate shift. \hfill \qed
\end{example}

\begin{example}[Covariate shift does not imply SJS]\label{ex:notSJS}
Under Assumption~\ref{as:cont}, let $\ell = 2$ and $\mathcal{F}$ be a sub-$\sigma$-algebra of $\mathcal{H}$.
Suppose that $0 < P[A_1\,|\,\mathcal{H}] < 1$ $P$-almost surely and that
$\mathcal{F}$ is \textbf{not} sufficient for $\mathcal{H}$ with respect to $\mathcal{A}$ under $P$, i.e.\
we have
\begin{equation*}
1 \ > \ P\big[P[A_1\,|\,\mathcal{F}] = P[A_1\,|\,\mathcal{H}]\bigr].
\end{equation*}
By Theorem~1 of Tasche~\cite{tasche2022class} then there exists a probability measure $Q^\ast$ on $\bigl(\Omega,
\sigma(\mathcal{H}\cup\mathcal{A})\bigr)$ with $Q^\ast \ll P$ such that $Q^\ast$ and $P$ are related
through $\mathcal{H}$-covariate shift but not through $\mathcal{F}$-covariate shift. However, if $Q^\ast$ and $P$
were related through $\mathcal{F}$-SJS, then Lemma~\ref{le:covariateShift} from Appendix~\ref{se:proofs}
would imply that after all $Q^\ast$ and $P$
are related through $\mathcal{F}$-covariate shift. But they are not, by choice of $Q^\ast$.\\
Conclusion: For any sub-$\sigma$-algebra $\mathcal{F}
\subset \mathcal{H}$ which is not sufficient for $\mathcal{H}$, there is a $Q^\ast$ such that $Q^\ast$ and
$P$ are related
through $\mathcal{H}$-covariate shift but not through $\mathcal{F}$-SJS. \hfill \qed
\end{example}

For the interpretation of the following theorem, recall that by Lemma~\ref{le:measurable},
conditional distribution invariance as introduced in
Definition~\ref{de:CDI} is equivalent to the property that the target unconditional
feature distribution $Q|\mathcal{H}$ has got an $\mathcal{F}$-measurable
density with respect to the source unconditional feature distribution $P|\mathcal{H}$.

\begin{theorem}\label{th:SJSvsCovShift}
Under Assumption~\ref{as:cont}, let $\mathcal{F}$ be a sub-$\sigma$-algebra of $\mathcal{H}$.
Consider then the following three properties:
\begin{itemize}
\item[(SJS)] $Q$ is related to $P$ through $\mathcal{F}$-SJS.
\item[(CDI)] $Q$ and $P$ are related through $\mathcal{F}$-conditional distribution invariance on
$\mathcal{H}$.
\item[(CSH)] $Q$ and $P$ are related through $\mathcal{H}$-covariate shift in the sense of
Definition~\ref{de:covshift}.
\end{itemize}
The following statements relating to these properties hold true:
\begin{itemize}
\item[(i)] Suppose that $X_1, \ldots, X_\ell$ are $\mathcal{H}$-measurable non-negative random variables
and
define the random matrix $R=R(P, \mathcal{F}, \ell)$ by \eqref{eq:defR}, i.e.\
with $\mu = P$, $\mathcal{G} = \mathcal{F}$ and $n = \ell$.
If the matrix $R$ has full rank, i.e.\ it satisfies condition \eqref{eq:rank},
then  (CSH) follows if (SJS) and (CDI) are both true.
\item[(ii)] (SJS) follows if (CDI) and (CSH) are both true.
\item[(iii)] Suppose that the following condition is true:
\begin{equation}\label{eq:positive}
P\bigl[P[A_i\,|\,\mathcal{H}] > 0\bigr]\ = \ 1, \qquad\text{for all}\ i = 1,\ldots, \ell.
\end{equation}
Then (CDI) follows if (SJS) and (CSH) are both true.
\end{itemize}
\end{theorem}

See Appendix~\ref{se:proofs} for a proof of Theorem~\ref{th:SJSvsCovShift}.
As a first consequence of Theorem~\ref{th:SJSvsCovShift}, we find that in general the property `conditional
distribution invariance' does imply neither covariate shift nor SJS, as demonstrated in the following
example.

\begin{example}\label{ex:CDInotSJS}
With $P$, $\mathcal{F}$ and $\ell =2$ as in Assumption~\ref{as:setting},
define the probability measure $Q$ on $\mathcal{H}$ by taking recourse
to Lemma~\ref{le:measurable}, with some $\mathcal{F}$-measurable density
$f = \frac{d Q|\mathcal{H}}{d P|\mathcal{H}}$. Then extend $Q$ to $\sigma(\mathcal{H}\cup\mathcal{A})$ by
letting $Q[A_1\,|\,\mathcal{H}] = \bigl(P[A_1\,|\,\mathcal{H}]\bigr)^2$.

By construction, then $Q$ and $P$ are related through $\mathcal{F}$-conditional distribution
invariance on $\mathcal{H}$
but not related to each other through $\mathcal{H}$-covariate shift. Therefore, thanks to Lemma~\ref{le:binary} and
Theorem~\ref{th:SJSvsCovShift}~(i), $Q$ and $P$ cannot be related through $\mathcal{F}$-SJS either if $\mathcal{F}$
is not sufficient for $\mathcal{H}$ with respect to $\mathcal{A}$. \hfill \qed
\end{example}

What about the practical relevance of Theorem~\ref{th:SJSvsCovShift}? Condition~(CDI) might be most interesting.
As mentioned before, (CDI)
is an observable property, in the sense that it depends on the feature marginal distributions $Q|\mathcal{H}$ and
$P|\mathcal{H}$ only which both are assumed to be observed.
Hence if (CDI) is found to occur then by
Theorem~\ref{th:SJSvsCovShift}~(i) and (ii), SJS and covariate shift are not mutually exclusive alternatives but
either they both hold or neither of the two does hold. Thus if it can be assumed for whatever reasons
that source and target are related through SJS or covariate shift, covariate shift might be preferred because
it is easier to handle.

\begin{remark}
Chen et al.~\cite{chen&zaharia&Zou:SJS} defined in Definition~3 \emph{$m$-sparse covariate shift} by the property
(adapted to the setting of Section~\ref{se:setting})
\begin{equation}\label{eq:msparsedef}
Q[A_i \cap H\,|\,\mathcal{F}]\  =\ P[A_i\cap H\,|\,\mathcal{F}], \qquad
H \in \mathcal{H}, i = 1, \ldots, \ell,
\end{equation}
for some sub-$\sigma$-algebra $\mathcal{F}$ of $\mathcal{H}$. In the context of this paper, property
\eqref{eq:msparsedef} rather would be called \emph{$\mathcal{F}$-sparse covariate shift}.
In their Theorem~2, Chen et al.~\cite{chen&zaharia&Zou:SJS} correctly
claimed that \eqref{eq:msparsedef} implies $\mathcal{F}$-SJS, as can also be seen from
the proof of Theorem~\ref{th:SJSvsCovShift}~(ii) in Appendix~\ref{se:proofs}.
In their proof of Theorem~2, however,
Chen et al.~only proved that $\mathcal{H}$-covariate shift and
$\mathcal{F}$-conditional distribution invariance on $\mathcal{H}$ together imply \eqref{eq:msparsedef}.
Conversely, by summing over $i = 1, \ldots, \ell$, \eqref{eq:msparsedef} can be seen to
imply $\mathcal{F}$-conditional distribution invariance on $\mathcal{H}$.
These observations are summarized in the following two statements:
\begin{itemize}
\item[(i)] $\mathcal{H}$-covariate shift and $\mathcal{F}$-conditional distribution invariance on $\mathcal{H}$
together imply $\mathcal{F}$-sparse covariate shift as defined by \eqref{eq:msparsedef}.
\item[(ii)] $\mathcal{F}$-sparse covariate shift implies $\mathcal{F}$-SJS and
$\mathcal{F}$-conditional distribution invariance on $\mathcal{H}$.
\end{itemize}
If the source distribution $P$ and the target distribution $Q$ are related through
$\mathcal{F}$-sparse covariate shift and
the rank condition of Theorem~\ref{th:SJSvsCovShift}~(i) is satisfied, then statement~(ii) and
Theorem~\ref{th:SJSvsCovShift}~(i) imply that $P$ and $Q$ are also related through $\mathcal{H}$-covariate shift.
Consequently, `sparse covariate shift' might be broadly described by the `equation'
\begin{align*}
\textit{sparse\,covariate\,shift} & = \textit{covariate\,shift}\, \cap\,
    \textit{conditional\,distribution\,invariance}\\
& = \textit{sparse\,joint\,shift}\, \cap\, \textit{conditional\,distribution\,invariance}.\qquad\Box
\end{align*}
\end{remark}

\section{How to estimate sparse joint shift?}
\label{se:how}

Chen et al.~\cite{chen&zaharia&Zou:SJS} proposed two estimation approaches for the weight function that
corresponds to the density $\bar{h} = \frac{d Q}{d P}$ in the setting of this paper.
In the following, we begin with discussing their estimation strategy in general in Section~\ref{se:general}.
Then we continue with descriptions of their approach SEES-c in Section~\ref{se:KL}
and of their approach SEES-d in Section~\ref{se:discrete} below,
and show how these approaches can be derived in the setting of this article.

\subsection{The general approach}
\label{se:general}

The general estimation approach proposed by Chen et al.~\cite{chen&zaharia&Zou:SJS} can be described this way:
Approximate the observed feature density $\widehat{h}$
(on the target dataset) by the hypothetical density $h_\mathcal{F}$ it would have under $\mathcal{F}$-SJS.
In the following, we present this approach in mathematical terms.
\begin{subequations}
In the context of Assumption~\ref{as:cont}, let $\mathcal{F}$ be a sub-$\sigma$-algebra of $\mathcal{H}$
and keep the source distribution $P$ fixed. Define
\begin{equation}\label{eq:SF}
S(\mathcal{F})\ = \ \left\{h_Q = \frac{d Q|\mathcal{H}}{d P|\mathcal{H}}: Q \ll P, \ Q\ \text{related to}\ P
    \ \text{through}\ \mathcal{F}\text{-SJS}\right\}.
\end{equation}
In the following, we make use of the fact that
as a consequence of \eqref{eq:SJSdens} densities $h_Q\in S(\mathcal{F})$ can be represented as
\begin{equation}\label{eq:Hdens}
h_Q = \sum_{i=1}^\ell f_i\,\frac{Q[A_i]}{P[A_i]}\,P[A_i\,|\,\mathcal{H}],
\end{equation}
where $f_i$, $i = 1, \ldots, \ell$, is an $\mathcal{F}$-measurable density of the target class-conditional
feature distribution $Q_i$ with respect to the source class-conditional feature distribution $P_i$
(see Definition~\ref{de:condDist}) on $\mathcal{H}$.
\end{subequations}

Assume that $\widehat{h}$ is a density of the observed target feature distribution $\widehat{Q}|\mathcal{H}$
with respect to the (also known) source feature distribution $P|\mathcal{H}$.
The task is to approximate $\widehat{h}$ as best as possible by an element $h_Q$ of $S(\mathcal{F})$. Hence
the following optimisation problem is to be solved:
\begin{equation}\label{eq:optim}
h_Q^\ast \ = \ \arg\min\limits_{h_Q \in S(\mathcal{F})} D(\widehat{h}, h_Q),
\end{equation}
for some distance measure $D$ for densities. Together with \eqref{eq:Hdens}, this is the representation
of Eq.~(4.1) of Chen et al.~\cite{chen&zaharia&Zou:SJS} translated into the notions and notation of this paper.
An optimal density $h_Q^\ast$ is specified by the ingredients $f_i$, $Q[A_i]$, $P[A_i]$
and $P[A_i\,|\,\mathcal{H}]$, $i = 1, \ldots, \ell$, of \eqref{eq:Hdens}. The $P[A_i]$ and $P[A_i\,|\,\mathcal{H}]$
are supposed to be known from the source dataset. In any case, once the $f_i$ and the $Q[A_i]$
have been determined, the full target distribution density $\bar{h}_Q$ on $\sigma(\mathcal{H}\cup \mathcal{A})$
is given by the right-hand side of \eqref{eq:SJSdens}.

\subsection{Minimising the Kullback-Leibler divergence}
\label{se:KL}

Chen et al.~\cite{chen&zaharia&Zou:SJS} made two proposals for the choice of $D$ in \eqref{eq:optim}. The first
proposal, called SEES-c, is to make use of the Kullback-Leibler (KL) divergence.
\begin{definition}\label{de:KL}
Assume that $\pi$ and $\kappa$ are probability measures on a measurable space $(\Omega, \mathcal{M})$,
with densities $p$ and
$q$ respectively with regard to another measure $\mu$ on $(\Omega, \mathcal{M})$. Then the \emph{KL divergence} of
$p$ and $q$ is defined as
\begin{equation*}
D_{KL}(p\Vert q) \ = \ \int p\,\log\left(\frac{p}{q}\right) d\mu \ = \ E_\pi\left[\log\left(\frac{p}{q}\right)\right],
\end{equation*}
with $0\,\log(0) \stackrel{\mathrm{def}}{=} 0$ and $D_{KL}(p\Vert q) = \infty$
if $\kappa(M) = 0 < \pi(M)$ for any $M\in \mathcal{M}$ (i.e.\ if $\pi$ is not
absolutely continuous with respect to $\kappa$).
\end{definition}
\begin{subequations}
Note that in general $D_{KL}(p\Vert q) \neq D_{KL}(q\Vert p)$. By Jensen's inequality,
 $D_{KL}(p\Vert q)$ is non-negative and equals $0$ if and only if $\pi$-almost surely it holds that $p = q$.

Writing down $D_{KL}$ with $\mu = P$, $\pi =\widehat{Q}|\mathcal{H}$, as well as
$\kappa = Q|\mathcal{H}$ and making use of \eqref{eq:Hdens} gives
\begin{align}
D_{KL}(\widehat{h}\Vert h_Q) & \ = \ E_{\widehat{Q}}\left[\log\left(\frac{\widehat{h}}
    {\sum_{i=1}^\ell f_i\,\frac{Q[A_i]}{P[A_i]}\,P[A_i\,|\,\mathcal{H}]}\right)\right]\notag\\
   & \ = \ E_{\widehat{Q}}[\log(\widehat{h})] -
   E_{\widehat{Q}}\left[\log\left(\sum\nolimits_{i=1}^\ell f_i\,\tfrac{Q[A_i]}{P[A_i]}\,
       P[A_i\,|\,\mathcal{H}]\right)\right].\label{eq:KLopt}
\end{align}
On the right-hand side of \eqref{eq:KLopt}, when optimisation over $h_Q$ is intended,
$E_{\widehat{Q}}[\log(\widehat{h})]$ is a constant and
can be ignored. As a consequence, estimation of $\widehat{h}$ in \eqref{eq:optim} can be avoided
for $D = D_{KL}$. In the second term on the right-hand side of \eqref{eq:KLopt},
the $P[A_i]$ and $P[A_i\,|\,\mathcal{H}]$ are observable and can be
treated as known quantities such that the $f_i$ and $Q[A_i]$ are left as variables whose values are to be optimally
chosen.

If $\widehat{\varphi}_i$ denotes an estimator for $f_i\,Q[A_i]$, the $\widehat{\varphi}_i$ must be
$\mathcal{F}$-measurable and satisfy the constraints
\begin{equation}\label{eq:QiA}
E_{P_i}[\widehat{\varphi}_i] \ = \ Q[A_i], \qquad i = 1, \ldots, \ell.
\end{equation}
This implies
\begin{equation}\label{eq:cond}
1 \ = \ \sum_{i=1}^\ell E_{P_i}[\widehat{\varphi}_i]  \ = \ E_P\left[\sum_{i=1}^\ell \widehat{\varphi}_i\,
    \frac{P[A_i\,|\,\mathcal{H}]}{P[A_i]}\right].
\end{equation}
\end{subequations}
Combining \eqref{eq:KLopt} and \eqref{eq:cond} for the specification of \eqref{eq:optim} with
$D = D_{KL}$ gives the optimisation problem
\begin{equation}\label{eq:optKL}
\begin{split}
\max\limits_{\widehat{\varphi}_1, \ldots, \widehat{\varphi}_\ell \ge 0,\, \mathcal{F}\text{-measurable}}
    E_{\widehat{Q}}\left[\log\left(\sum_{i=1}^\ell \widehat{\varphi}_i\,
    \tfrac{P[A_i\,|\,\mathcal{H}]}{P[A_i]}\right)\right], & \\
   \text{subject to:}\quad E_P\left[\sum_{i=1}^\ell \widehat{\varphi}_i\,
    \frac{P[A_i\,|\,\mathcal{H}]}{P[A_i]}\right] \ = \ 1. &
\end{split}
\end{equation}
Observe that in \eqref{eq:optKL} the expectation $E_P$ can be approximated as a sample mean on
the source dataset while $E_{\widehat{Q}}$ can be approximated as a sample mean on the target dataset.
Once optimal $\varphi_1^\ast, \ldots, \varphi_\ell^\ast$ have been found, estimates of the target
class prior probabilities $Q[A_i]$ can be determined by means of \eqref{eq:QiA}, with the expectations
$E_{P_i}$ replaced by sample means on the strata of the source dataset defined by the observed class labels.

\eqref{eq:optKL} corresponds to the optimisation objective for SEES-c of Chen et al.~\cite{chen&zaharia&Zou:SJS}
in the case where $m$ and $I$ in Definition~1 of Chen et al.~\cite{chen&zaharia&Zou:SJS} are fixed (case of fixed
shift index set).
In the paragraph `SEES-c', Chen et al.~\cite{chen&zaharia&Zou:SJS} suggest  keeping
$\mathcal{F}$ (i.e.\ the set of input variables
to the functions $\widehat{\varphi}_i$) variable in \eqref{eq:optKL} in order to be able to
find an `optimal' $\mathcal{F}$.
For achieving this, Chen et al.~add a term on the right-hand side of \eqref{eq:optKL} that
rewards solutions with low numbers of input features.

However, it is not clear that better fit by $\mathcal{F}$ than by $\mathcal{F}'$ implies that $\mathcal{F}$-SJS is
a more appropriate model than $\mathcal{F}'$-SJS. For instance, assuming straightforward covariate shift would
give an exact fit. Similarly, with `factorizable joint shift', exact fit could even be achieved for arbitrary target
label distributions (in the binary case, see Proposition~2 of Tasche~\cite{tasche2022factorizable}).
Moreover, Chen et al.~\cite{chen&zaharia&Zou:SJS} do not provide the details of how they identify the set of features
to explain the SJS once an optimal solution to their version of \eqref{eq:optKL} has been found.
Remark~\ref{rm:LeaveOneOut} above suggests that picking the features with the $d-1$ largest contributions
$\hat{\beta}_i$ as defined in Section~4.2 of Chen et al.~\cite{chen&zaharia&Zou:SJS} would
be a theoretically sound approach.

\subsection{Second estimation strategy}
\label{se:discrete}

The second proposal of Chen et al.~\cite{chen&zaharia&Zou:SJS} for implementing \eqref{eq:optim} is
intended to apply to the case where
all the features are discrete (or have been discretised). The proposal is not straightforward at first glance because
Chen et al.\ introduce -- somewhat unexpected -- in their Eq.~(4.2) an ``ML model'' $f: \mathbb{R}^d \to \mathcal{Y}$
which in the context of their SEES-d approach appears to be a hard classifier as in \eqref{eq:classifier}.

In abstract terms, the rationale behind the SEES-d approach of Chen et al.~\cite{chen&zaharia&Zou:SJS}
can be described as follows:

In practice, one issue with the optimisation problem \eqref{eq:optim} could be
that $\widehat{h}$ and the $P[A_i\,|\,\mathcal{H}]$
which appear in \eqref{eq:Hdens} are hard to estimate due to the curse of dimensionality when the
original feature information set $\mathcal{H}$ is big because it is generated by a high-dimensional feature
vector.

As a workaround for this problem, Chen et al.~\cite{chen&zaharia&Zou:SJS} suggest replacing the density
family $S(\mathcal{F})$ as defined by \eqref{eq:SF} with $S(\mathcal{F}, \mathcal{H}')$ defined
by
\begin{subequations}
\begin{equation}\label{eq:SFH}
S(\mathcal{F}, \mathcal{H}')\ = \ \left\{h_Q = \frac{d Q|\mathcal{H}'}{d P|\mathcal{H}'}: Q \ll P, \ Q\
\text{related to}\ P    \ \text{through}\ (\mathcal{F}, \mathcal{H}')\text{-SJS}\right\},
\end{equation}
where $\mathcal{H}' \supset \mathcal{F}$ is a smaller sub-$\sigma$-algebra of
$\mathcal{H}$. This could be a sub-$\sigma$-algebra generated by a lower-dimensional set of features.

\eqref{eq:optim} then is replaced with
\begin{equation}\label{eq:optimH}
h_Q^\ast \ = \ \arg\min\limits_{h_Q \in S(\mathcal{F}, \mathcal{H}')} D(\widehat{h}', h_Q),
\end{equation}
\end{subequations}
where $\widehat{h}'$ is a density of $\widehat{Q}|\mathcal{H}'$ with respect to
$P|\mathcal{H}'$ which can be represented as
\begin{equation}\label{eq:hprime}
\widehat{h}' \ = \ E_P[\widehat{h}\,|\,\mathcal{H}'].
\end{equation}
Assume that the probability measure $\widehat{Q}$ whose $\mathcal{H}$-density $\widehat{h}$
appears in \eqref{eq:optim} and $P$ are related through $\mathcal{F}$-SJS. Then it follows that
both $\min_{h_Q \in S(\mathcal{F})} D(\widehat{h}, h_Q) = 0$ and
$\min_{h_Q \in S(\mathcal{F}, \mathcal{H}')} D(\widehat{h}', h_Q) = 0$. In the first case the interesting
(not necessarily unique) minimising density $\widehat{h} = h_Q^\ast$ has the representation
\begin{equation*}
\widehat{h} = \sum_{i=1}^\ell \widehat{f}_i\,\frac{\widehat{Q}[A_i]}{P[A_i]}\,P[A_i\,|\,\mathcal{H}],
\end{equation*}
while in the second case, thanks to \eqref{eq:hprime} the $\mathcal{H}'$-density
$\widehat{h}'$ of $\widehat{Q}$ can be written as
\begin{equation*}
\widehat{h}' = \sum_{i=1}^\ell \widehat{f}_i\,\frac{\widehat{Q}[A_i]}{P[A_i]}\,P[A_i\,|\,\mathcal{H}'].
\end{equation*}
Hence if both minimising densities are unique (as for instance when Theorem~\ref{th:ident} is applicable),
they give rise to the same sets of $\widehat{f}_i$ and $\widehat{Q}[A_i]$ such that the full
$\sigma(\mathcal{H}\cup \mathcal{A})$-density as given by the right-hand side of \eqref{eq:SJSdens}
is well-defined and unique.

Unfortunately, if $\mathcal{H'}$ in \eqref{eq:SFH} is chosen too small the solution $h_Q^\ast$
may no longer be unique since
satisfying the rank condition \eqref{eq:rank} of Theorem~\ref{th:ident} might turn out to be impossible.

Chen et al.~\cite{chen&zaharia&Zou:SJS} deploy another workaround to deal with this side effect. Let
$\mathbf{C} = (C_1, \ldots, C_\ell)$ be a hard (or crisp) classifier in the sense of \eqref{eq:classifier}.
As mentioned before, Chen et al.~\cite{chen&zaharia&Zou:SJS} talk of ``ML model''
instead of classifier and denote it by $f$.
Define $\mathcal{C} = \sigma(\mathbf{C})= \sigma\bigl(\{C_1, \ldots, C_\ell\}\bigr)$.
Chen et al.~\cite{chen&zaharia&Zou:SJS}
then solve \eqref{eq:optimH} with $\mathcal{H}_\mathcal{C} = \sigma(\mathcal{H}' \cup \mathcal{C})$ substituted for
$\mathcal{H}'$. With this augmentation, if $\mathbf{C}$ is a reasonably accurate classifier, it becomes likely
that in Theorem~\ref{th:ident} the rank condition \eqref{eq:rank} can be satisfied with the choice
$(X_1, \ldots, X_\ell) = \mathbf{C}$.

Along this line of thought, also the choice $\mathcal{H}' = \sigma(\mathcal{F} \cup \mathcal{C})$ might make
sense in \eqref{eq:optimH}. Via Eq.~\eqref{eq:obtain} in Appendix~\ref{se:proofs} below,
with $\mathcal{G}= \mathcal{F}$ and $X_i = \mathbf{1}_{C_i}$,
this choice leads to a \emph{conditional} (on $\mathcal{F}$) \emph{confusion matrix} method for the posterior
probabilities in the target domain for general SJS and the original confusion matrix method
(Saerens et al.~\cite{saerens2002adjusting}) for prior probability shift (case $\mathcal{F} = \{\emptyset, \Omega\}$).
See \eqref{eq:linearF} for details.

A full `instantiation' of the approach sketched in this sub-section  is detailed for the readers' convenience
in Appendix~\ref{se:instant} for the case where all features are discrete.

\begin{remark}
In their implementation of SEES-d, Chen et al.~\cite{chen&zaharia&Zou:SJS} chose the
$\mathcal{H}'$ in \eqref{eq:optimH}
dependent on $\mathcal{F}$, i.e.\ $\mathcal{H}' = \Phi(\mathcal{F})$ for some algorithm $\Phi$, in order
to be able to select an `optimal' $\mathcal{F}$ by comparing the minima achieved in \eqref{eq:optimH}
for different $\mathcal{F}$. As different $\mathcal{H}'$ may result in different marginal densities
$\widehat{h}'$, there is a risk of comparing apples to pears when the originally observed marginal
feature density $\widehat{h}$ is not an element of $S(\mathcal{F})$ as defined in \eqref{eq:SF}.
As a consequence, the true minimum of \eqref{eq:optim} might be missed.\hfill \qed
\end{remark}

As mentioned before at the end of Section~\ref{se:KL}, taking recourse to Corollary~\ref{co:all} above,
perhaps along the lines suggested in Remark~\ref{rm:LeaveOneOut},
could provide an alternative approach to reduce the computational complexity of solving the opimisation problems
\eqref{eq:optim} and \eqref{eq:optimH}.

\section{Conclusions}

Modelling dataset shift with sparse joint shift (SJS) as proposed by Chen et al.~\cite{chen&zaharia&Zou:SJS}
is a promising approach to inferring the full joint distribution of features and labels of
a target dataset for which only the features have been observed. The approach is based on learning from
a source dataset with complete information on features and labels. It represents a proper alternative to
the popular inference approaches based on assumptions of prior probability and covariate shift respectively.

In this paper, we have revisited and complemented the theory for SJS in some aspects, in particular
with regard to the transmission of SJS from sets of features to larger sets of features, a conditional
correction formula for the class posterior probabilities, identifiability of SJS, and
the relationship between SJS and covariate shift.

In addition, with regard to the estimation of the shift characteristics in practice, we have
complemented the rationale for the algorithms proposed by Chen et al.~\cite{chen&zaharia&Zou:SJS} and pointed out
inconsistencies which could prevent them from identifying optimal solutions, while at the same
time suggesting improvements to avoid these issues.

\section*{Acknowledgments}
The author thanks an anonymous reviewer for suggestions that helped to
improve an earlier version of this article.

\appendix

\section{Some concepts and notation from probability theory}
\label{se:notation}

The general approach to the subject in this paper is based on application of
probabilistic measure theory, as described for instance in the textbooks of
Billingsley~\cite{billingsley1986probability} and Klenke~\cite{klenke2013probability}.

For quick reference, here is a brief glossary of concepts and notation frequently used
in the paper:

\textbf{Measurable space $(\Omega, \mathcal{M})$.} According to Billingsley~\cite{billingsley1986probability} (p.~16)
``In probability theory $\Omega$ consists of all the possible results or outcomes
$\omega$ of an experiment or observation.'' $\mathcal{M}$ is a $\sigma$-\emph{algebra}
(also called $\sigma$-field), i.e.\ a family of
subsets of $\Omega$ with the following three properties:
\begin{itemize}
\item[(i)] $\emptyset \in \mathcal{M}$,
\item[(ii)] $M \in \mathcal{M} \Rightarrow \Omega \setminus M \in \mathcal{M}$.
\item[(iii)] $M_n \in \mathcal{M}$ for $n \in \mathbb{N}$ $\Rightarrow$ $\bigcup_{n\in\mathbb{N}} M_n \in \mathcal{M}$.
\end{itemize}
The intersections of $\sigma$-algebras again are $\sigma$-algebras. Is $\mathcal{M}_0$ is any family of
subsets of $\Omega$, the (non-empty) intersection of all $\sigma$-algebras on $\Omega$ that contain
$\mathcal{M}_0$ is called $\sigma$-\emph{algebra generated by} $\mathcal{M}_0$ and denoted by $\sigma(\mathcal{M}_0)$.
The $\sigma$-algebra $\sigma(\mathcal{M}_0)$ is also the smallest $\sigma$-algebra containing $\mathcal{M}_0$.

\textbf{Probability space $(\Omega, \mathcal{M}, P)$.} $(\Omega, \mathcal{M})$ is a measurable space,
$P$ is a \emph{probability measure} on $(\Omega, \mathcal{M})$, i.e.\ a measure
(see Chapter~2 of Billingsley~\cite{billingsley1986probability}) with $P[\Omega] = 1$.
For $M\in \mathcal{M}$, $P[M]$ is the
\emph{probability of the event} $M$. A property of outcomes $\omega\in\Omega$ is said
to hold $P$-\emph{almost surely} if
the event of all the $\omega$ with the property has probability $1$ under $P$.

\textbf{Mappings and random variables.} For a fixed set $S$, the function $\mathbf{1}_S$ with
$\mathbf{1}_S(s) = 1$ for $s \in S$ and
$\mathbf{1}_S(s) = 0$ for $s \notin S$ is called \emph{indicator function}.

A mapping $X: (\Omega, \mathcal{M}) \to (\Omega', \mathcal{M}')$ is $\mathcal{M}'$-$\mathcal{M}$-\emph{measurable}
if $X^{-1}(M') \in \mathcal{M}$ for all $M' \in \mathcal{M}'$.

A real-valued measurable mapping $X$
defined on a probability space $(\Omega, \mathcal{M}, P)$ is also called \emph{random variable} or \emph{statistic}.

\textbf{Information set.} $\sigma$-algebras, on the one hand, are needed for consistent definitions of
integrals and probabilities. On the other hand, they are of interest because they can be used to reflect
available information. If $\mathcal{M}_1, \mathcal{M}_2$ are $\sigma$-algebras with $\mathcal{M}_1 \subset
\mathcal{M}_2$ then $\mathcal{M}_1$ conveys less information than $\mathcal{M}_2$. If $X$ is a mapping
into a measurable space $(\Omega', \mathcal{M}')$, the $\sigma$-algebra $\sigma(X) =
\sigma(X^{-1}(M'): M'\in \mathcal{M})$ reflects the information obtainable by observation of the outcomes of $X$.
For these reasons, $\sigma$-algebras are also called \emph{information sets}
(see, e.g., Holzmann and Eulert~\cite{HolzmannInformation2014}).

\textbf{Expected value and conditional expectation.} If $(\Omega, \mathcal{M}, P)$ is a probability space and $X$ an
integrable random variable
on $\Omega$, the \emph{expected value} of $X$ is $E_P[X] = \int X\,d P$. 

If $\mathcal{H}\subset \mathcal{M}$ is a sub-$\sigma$-algebra of $\mathcal{M}$, any $\mathcal{H}$-measurable
integrable random variable $Z$ is called \emph{expectation of} $X$ \emph{conditional on} $\mathcal{H}$ if
$E_P[\mathbf{1}_H\,Z] = E_P[\mathbf{1}_H\,X]$ for all $H \in \mathcal{H}$. If $Z_1$ and $Z_2$ are both
expectations of $X$ conditional on $\mathcal{H}$ it follows that $P[Z_1 \neq Z_2] = 0$. $E_P[X\,|\,\mathcal{H}]$ is
short hand for any expectation $Z$ of $X$ conditional on $\mathcal{H}$.

For any $M\in \mathcal{M}$, the conditional expectation $E_P[\mathbf{1}_M\,|\,\mathcal{H}] = P[M\,|\,\mathcal{H}]$
is called \emph{probability of} $M$ \emph{conditional on} $\mathcal{H}$.

\textbf{Absolute continuity and densities.} Let $(\Omega, \mathcal{M})$ be a measurable space and $\mu$ and $\nu$
be measures on $(\Omega, \mathcal{M})$. The measure $\mu$ is \emph{absolutely continuous with respect to} $\nu$
(short hand: $\mu \ll \nu$) if $\nu[N] = 0$ for any $N \in \mathcal{M}$ implies $\mu[N] = 0$.

By the Radon-Nikodym theorem (Theorem~32.2 of Billingsley~\cite{billingsley1986probability}),
if $\mu$ and $\nu$ are both
$\sigma$-finite then $\mu \ll \nu$ is equivalent to the existence of a density $\varphi = \frac{d \mu}{d \nu}$
of $\mu$ with respect to $\nu$, i.e.\ of a non-negative $\mathcal{M}$-measurable function $\varphi$ such that
$\mu[M] = \int_M \varphi\,d \nu$ for all $M\in \mathcal{M}$.

Let $Q \ll P$ be probability measures on $(\Omega, \mathcal{M})$ with density $\varphi = \frac{d Q}{d P}$ and
$\mathcal{H}$ be a sub-$\sigma$-algebra of $\mathcal{M}$. Then $\varphi_\mathcal{H} = E_P[\varphi\,|\,\mathcal{H}]$
is a density of $Q$ with respect to $P$ on the measurable space $(\Omega, \mathcal{H})$. In this case, the
application of the operator $E_P[\cdot\,|\,\mathcal{H}]$ corresponds to the `integrating out a set of variables'
operation in the setting of multivariate Lebesgue-densities in order to obtain a marginal distribution.

We use short hand $P|\mathcal{H}$ ($P$ restricted to $\mathcal{H}$) for the probability space
$(\Omega, \mathcal{H}, P)$.

\section{Proofs}
\label{se:proofs}

\begin{proof}[\textbf{Proof of Proposition~\ref{pr:properties}.}]
Since (i) implies (iii) with $\mathcal{G} = \mathcal{H}$,
we only need to prove that \eqref{eq:X} holds if \eqref{eq:restrictSJS} is true.

Observe that under Assumption~\ref{as:cont}, $Q$ is also absolutely continuous with respect to $P$ on $\mathcal{H}$.
Moreover, by Lemma~1 of Tasche~\cite{tasche2022factorizable} each target
class-conditional feature distribution $Q_i$ is also
absolutely continuous with respect to the source class-conditional feature distribution $P_i$ on $\mathcal{H}$.

Let $G \in \mathcal{G}$ and
$F \in\mathcal{F}$ as well as $i \in \{1, \ldots, \ell\}$
be given. Denote by $h_i$ any $\mathcal{H}$-density of $Q_i$ with respect to $P_i$.
Then $f_i = E_{P_i}[h _i\,|\,\mathcal{F}]$ is $\mathcal{F}$-measurable and a density of $Q_i$ with respect to $P_i$
on $\mathcal{F}$.
Therefore, we obtain for any $\bar{F} \in \mathcal{F}$
\begin{align*}
E_{Q_i}\bigl[\mathbf{1}_{\bar{F}}\,P_i[G\cap F\,|\,\mathcal{F}]\bigr] & =
    E_{P_i}\bigl[f_i\,\mathbf{1}_{\bar{F}}\,P_i[G\cap F\,|\,\mathcal{F}]\bigr]\\
    & = E_{P_i}\bigl[f_i\,\mathbf{1}_{\bar{F}\cap F}\,P_i[G\,|\,\mathcal{F}]\bigr]\\
    & = E_{P_i}\bigl[f_i\,\mathbf{1}_{\bar{F}\cap F}\,Q_i[G\,|\,\mathcal{F}]\bigr]\\
    & = E_{Q_i}\bigl[\mathbf{1}_{\bar{F}\cap F}\,Q_i[G\,|\,\mathcal{F}]\bigr]\\
    & = Q_i[\bar{F}\cap F\cap G].
\end{align*}
This proves \eqref{eq:SJS} for
$H \in \mathfrak{P}=\{F\cap G: F\in\mathcal{F}, G \in \mathcal{G}\} \subset
\mathcal{H}$.
The set family $\mathfrak{P}$ is $\cap$-stable. In addition, it can readily be shown that the
set family of all $H\in\mathcal{H}$ such that \eqref{eq:SJS} is true is an $\lambda$-system
(or Dynkin system) in the sense of (Billingsley~\cite{billingsley1986probability}, p.~36).
Since $\mathcal{H} = \sigma(\mathfrak{P})$, it follows from Dynkin's $\pi$-$\lambda$-theorem
(Billingsley~\cite{billingsley1986probability}, Theorem~3.6) that $\eqref{eq:X}$ is true for $X = \mathbf{1}_H$ if
$H \in \mathcal{H}$. This implies that $\eqref{eq:X}$ holds for any $\mathcal{H}$-measurable
simple random variable, and by monotone convergence even for all non-negative $\mathcal{H}$-measurable
random variables (Billingsley~\cite{billingsley1986probability}, Theorems~13.5 and 34.2).
\end{proof}

\begin{proof}[\textbf{Proof of Proposition~\ref{pr:corrH}}]
By Corollary~1 of Tasche~\cite{tasche2022factorizable}, we have the density
$f =\sum_{j=1}^\ell f_j\,\frac{Q[A_j]}{P[A_j]}\,P[A_j\,|\,\mathcal{F}]$
of $Q$ with respect to $P$ on $\mathcal{F}$, where $f_j$ denotes a density of $Q_j$ with respect to
$P_j$ on $\mathcal{F}$. As a consequence, it holds that $Q[f>0] = 1$. Corollary~2 of
Tasche~\cite{tasche2022factorizable} then implies that $Q$-almost surely
\begin{equation}\label{eq:fi}
Q[A_i\,|\,\mathcal{F}] \ =\ f_i\,\frac{Q[A_i]}{P[A_i]}\,P[A_i\,|\,\mathcal{F}]\,/\,f, \quad
    i = 1, \ldots, \ell.
\end{equation}
Define
\begin{equation*}
N \ = \ \bigcup_{j=1}^\ell \bigl\{P[A_j\,|\,\mathcal{F}]=0, P[A_j\,|\,\mathcal{H}]>0\bigr\}.
\end{equation*}
Then it follows that $Q[N] = 0$ because
\begin{align*}
E_P\bigl[P[A_j\,|\,\mathcal{H}]\,\mathbf{1}_{\{P[A_j\,|\,\mathcal{F}]=0\}}\bigr] & =
    P\bigl[A_j \cap \bigl\{P[A_j\,|\,\mathcal{F}]=0\bigr\}\bigr]\\
    & = E_P\bigl[P[A_j\,|\,\mathcal{F}]\,\mathbf{1}_{\{P[A_j\,|\,\mathcal{F}]=0\}}\bigr]\\
    & = 0,
\end{align*}
and therefore $0 = P\bigl[P[A_j\,|\,\mathcal{F}]=0, P[A_j\,|\,\mathcal{H}]>0\bigr]$ and also
\begin{equation*}
0\ =\ Q\bigl[P[A_j\,|\,\mathcal{F}]=0, P[A_j\,|\,\mathcal{H}]>0\bigr].
\end{equation*}
With the convention $0 / 0 = 0$, therefore \eqref{eq:fi} implies for $i=1, \ldots, \ell$ that
\begin{equation}\label{eq:replace}
f\,\frac{Q[A_i\,|\,\mathcal{F}]}{P[A_i\,|\,\mathcal{F}]}\,P[A_i\,|\,\mathcal{H}] \ =\
    f_i\,\frac{Q[A_i]}{P[A_i]}\,P[A_i\,|\,\mathcal{H}]
\end{equation}
holds on $\Omega \setminus N$ and hence $Q$-almost surely.

Theorem~\ref{th:eqSJS} implies that
\begin{equation*}
h \  =\ \sum_{j=1}^\ell f_j\,\frac{Q[A_j]}{P[A_j]}\,P[A_j\,|\,\mathcal{H}]
\end{equation*}
is a density of $Q$ with respect to $P$ on $\mathcal{H}$ and $Q[h>0]=1$ holds true. In this case,
Corollary~2 of Tasche~\cite{tasche2022factorizable} implies that $Q$-almost surely
\begin{equation*}
Q[A_i\,|\,\mathcal{H}] \ =\ \frac{f_i\,\frac{Q[A_i]}{P[A_i]}\,P[A_i\,|\,\mathcal{H}]}
 {\sum_{j=1}^\ell f_j\,\frac{Q[A_j]}{P[A_j]}\,P[A_j\,|\,\mathcal{H}]}.\quad
    i = 1, \ldots, \ell.
\end{equation*}
Substituting the left-hand side of \eqref{eq:replace} for the terms
$f_j\,\frac{Q[A_j]}{P[A_j]}\,P[A_j\,|\,\mathcal{H}]$, $j=1, \ldots, \ell$, now implies \eqref{eq:greater0}
and \eqref{eq:corrH}.
\end{proof}

\begin{proof}[\textbf{Proof of Lemma~\ref{le:total}}]
\eqref{eq:matrix} is implied by
\begin{align*}
E_\mu[X\,|\,\mathcal{G}] & = \sum_{j=1}^\ell E_\mu[X\,\mathbf{1}_{A_j}\,|\,\mathcal{G}]\\
    & = \sum_{j=1}^\ell E_\mu\bigl[E_\mu[X\,\mathbf{1}_{A_j}\,|\,\sigma(\mathcal{G}\cup\mathcal{A})]\bigm|
        \mathcal{G}\bigr]\\
    & = \sum_{j=1}^\ell E_\mu\bigl[\mathbf{1}_{A_j}\,E_\mu[X\,|\,\sigma(\mathcal{G}\cup\mathcal{A})]
        \bigm| \mathcal{G}\bigr]\\
    & \stackrel{\eqref{eq:genBayes}}{=} \sum_{j=1}^\ell
        E_\mu\bigl[\mathbf{1}_{A_j}\,E_{\mu_j}[X\,|\,\mathcal{G}]\bigm| \mathcal{G}\bigr]\\
    & = \sum_{j=1}^\ell \mu[A_j\,|\,\mathcal{G}]\,E_{\mu_j}[X\,|\,\mathcal{G}],
\end{align*}
for any non-negative $\mathcal{M}$-measurable random variable $X$.
\end{proof}

\begin{proof}[\textbf{Proof of Theorem~\ref{th:ident}}] Define $\mathcal{G} = \sigma(\mathcal{F}\cup\mathcal{F}')$.
By Corollary~\ref{co:all}, both $Q$ and $Q'$ are related to $P$ through $\mathcal{G}$-SJS.
Hence for $i, j \in \{1, \ldots, \ell\}$ it holds that
\begin{equation*}
E_{P_j}[X_i\,|\,\mathcal{G}] \  = \ E_{Q_j}[X_i\,|\,\mathcal{G}] \ = \ E_{Q'_j}[X_i\,|\,\mathcal{G}].
\end{equation*}
Therefore, we have $R(P, \mathcal{G}, \ell) =
R(Q, \mathcal{G}, \ell) = R(Q', \mathcal{G}, \ell) = R$.
The matrix $R$ is invertible with probability $1$ under $P$ and therefore also under
$Q$ and $Q'$. Hence, because $Q$ and $Q'$ are equal on $\mathcal{H}$, we obtain from Lemma~\ref{le:total} that
\begin{equation}\label{eq:obtain}
\begin{pmatrix}
Q'[A_1\,|\,\mathcal{G}] \\
\vdots \\
Q'[A_\ell\,|\,\mathcal{G}]
\end{pmatrix}
\ = \
\begin{pmatrix}
Q[A_1\,|\,\mathcal{G}] \\
\vdots \\
Q[A_\ell\,|\,\mathcal{G}]
\end{pmatrix}
\ = \ R^{-1} \times
\begin{pmatrix}
E_Q[X_1\,|\,\mathcal{G}] \\
\vdots \\
E_Q[X_\ell\,|\,\mathcal{G}]
\end{pmatrix}.
\end{equation}
Proposition~\ref{pr:corrH} and \eqref{eq:obtain} together now imply that
\begin{equation}\label{eq:equal}
\frac{\frac{Q[A_i\,|\,\mathcal{G}]}{P[A_i\,|\,\mathcal{G}]} P[A_i\,|\,\mathcal{H}]}
{\sum_{j=1}^\ell \frac{Q[A_j\,|\,\mathcal{G}]}{P[A_j\,|\,\mathcal{G}]} P[A_j\,|\,\mathcal{H}]}
\ = \ Q[A_i\,|\,\mathcal{H}] \ = \ Q'[A_i\,|\,\mathcal{H}], \quad
    i = 1\,\ldots, \ell.
\end{equation}
By \eqref{eq:Hbar}, for each $M \in \sigma(\mathcal{H}\cup \mathcal{A})$ there are $H_1, \ldots, H_\ell$ such
that $M = \bigcup_{i=1}^\ell (H_i\cap A_i) $. Therefore, it follows from \eqref{eq:equal} that
\begin{align*}
Q'[M] & = \sum_{i=1}^\ell Q'[H_i \cap A_i]\\
    & = \sum_{i=1}^\ell E_{Q'}\bigl[\mathbf{1}_{H_i}\,Q'[A_i\,|\,\mathcal{H}]\bigr]\\
    & = \sum_{i=1}^\ell E_{Q'}\bigl[\mathbf{1}_{H_i}\,Q[A_i\,|\,\mathcal{H}]\bigr]\\
    & = \sum_{i=1}^\ell E_{Q}\bigl[\mathbf{1}_{H_i}\,Q[A_i\,|\,\mathcal{H}]\bigr]\\
    & = \sum_{i=1}^\ell Q[H_i \cap A_i]\\
    & = Q[M].
\end{align*}
This completes the proof.
\end{proof}

\begin{proof}[\textbf{Proof of Lemma~\ref{le:binary}}]
Define $\mathcal{G} =  \sigma(\mathcal{F}\cup\mathcal{F}')$. We have to show that
\begin{equation*}
P\bigl[\det\bigl(R(P, \mathcal{G}, 2)\bigr) \neq 0\bigr]\ =\ 1
\end{equation*}
and \eqref{eq:binary} are equivalent. First, we obtain
\begin{align}
\det(R) & =  \begin{pmatrix}
E_{P_1}[X_1\,|\,\mathcal{G}] & E_{P_2}[X_1\,|\,\mathcal{G}]\\
E_{P_1}[X_2\,|\,\mathcal{G}] & E_{P_2}[X_2\,|\,\mathcal{G}]
\end{pmatrix} \notag\\
    & = \begin{pmatrix}
\frac{E_P[\mathbf{1}_{A_1}\,X_1\,|\,\mathcal{G}]}{P[A_1\,|\,\mathcal{G}]} &
\frac{E_P[\mathbf{1}_{A_2}\,X_1\,|\,\mathcal{G}]}{P[A_2\,|\,\mathcal{G}]}\\
\frac{E_P[\mathbf{1}_{A_1}\,X_2\,|\,\mathcal{G}]}{P[A_1\,|\,\mathcal{G}]} &
\frac{E_P[\mathbf{1}_{A_2}\,X_2\,|\,\mathcal{G}]}{P[A_2\,|\,\mathcal{G}]}
\end{pmatrix}. \label{eq:comment}
\end{align}
Note that in \eqref{eq:comment}, $P[A_j\,|\,\mathcal{G}] = 0$ implies
$E_P[\mathbf{1}_{A_j}\,X_i\,|\,\mathcal{G}] = 0$ and, therefore, ``$0/0$'' is understood as $0$.
It follows that
\begin{multline*}
\det(R) \neq 0 \\ \iff \quad
E_P[\mathbf{1}_{A_1}\,X_1\,|\,\mathcal{G}]\,E_P[\mathbf{1}_{A_2}\,X_2\,|\,\mathcal{G}] -
E_P[\mathbf{1}_{A_2}\,X_1\,|\,\mathcal{G}]\,E_P[\mathbf{1}_{A_1}\,X_2\,|\,\mathcal{G}] \neq 0.
\end{multline*}
Making now use of \eqref{eq:defX} and hence $X_2 = 1 - X_1$, we obtain
\begin{eqnarray*}
\lefteqn{E_P[\mathbf{1}_{A_1}\,X_1\,|\,\mathcal{G}]\,E_P[\mathbf{1}_{A_2}\,X_2\,|\,\mathcal{G}] -
E_P[\mathbf{1}_{A_2}\,X_1\,|\,\mathcal{G}]\,E_P[\mathbf{1}_{A_1}\,X_2\,|\,\mathcal{G}]}  \\
 & = & E_P\big[P[A_1\,|\,\mathcal{H}]^2\,\big|\, \mathcal{G}\big] \Big(P[A_2\,|\,\mathcal{G}] -
        E_P\big[P[A_1\,|\,\mathcal{H}]\,P[A_2\,|\,\mathcal{H}]\,\big|\, \mathcal{G}\big]\Big) \\
 & \qquad & - \Big(P[A_1\,|\,\mathcal{G}] -
        E_P\big[P[A_1\,|\,\mathcal{H}]^2\,\big|\, \mathcal{G}\big]\Big)
        E_P\big[P[A_1\,|\,\mathcal{H}]\,P[A_2\,|\,\mathcal{H}]\,\big|\, \mathcal{G}\big] \\
 & = & E_P\big[P[A_1\,|\,\mathcal{H}]^2\,\big|\, \mathcal{G}\big] \Big(1 -2\,P[A_1\,|\,\mathcal{G}] +
     E_P\big[P[A_1\,|\,\mathcal{H}]^2\,\big|\, \mathcal{G}\big]\Big)\\
 &  & \qquad\qquad
     -\ \Big(P[A_1\,|\,\mathcal{G}] -     E_P\Big[P[A_1\,|\,\mathcal{H}]^2\,\big|\, \mathcal{G}\Big]\Big)^2 \\
 & = & E_P\big[P[A_1\,|\,\mathcal{H}]^2\,\big|\, \mathcal{G}\big] - P[A_1\,|\,\mathcal{G}]^2 \\
 & = & \mathrm{var}_P\bigl[P[A_1\,|\,\mathcal{H}]\,\big|\,\mathcal{G}\bigr].
\end{eqnarray*}
As a result, we have found that
\begin{equation*}
\det(R) \neq 0 \qquad \iff \qquad \mathrm{var}_P\bigl[P[A_1\,|\,\mathcal{H}]\,\big|\,\mathcal{G}\bigr] >0.
\end{equation*}
Since $\mathrm{var}_P\bigl[P[A_1\,|\,\mathcal{H}]\,\big|\,\mathcal{G}\bigr] >0$ $\Leftrightarrow$
$P[A_1\,|\,\mathcal{H}] \neq P[A_1\,|\,\mathcal{G}]$ with positive probability,
the equivalence of \eqref{eq:rank} and \eqref{eq:binary} follows.
\end{proof}

\begin{lemma}\label{le:covariateShift}
Under Assumption~\ref{as:cont}, let $\mathcal{F}$ be a sub-$\sigma$-algebra of $\mathcal{H}$ such
that $Q$ is related to $P$ through $\mathcal{F}$-SJS. Suppose in addition that
$Q$ and $P$ are also related by $\mathcal{H}$-covariate shift (i.e.\ \eqref{eq:covshift} is true) and
that condition \eqref{eq:positive} is satisfied. Then it follows that
\begin{equation}\label{eq:hence}
Q[A_i\,|\,\mathcal{F}] \ = \ P[A_i\,|\,\mathcal{F}], \qquad i = 1, \ldots, \ell.
\end{equation}
\end{lemma}

\begin{proof}[Proof of Lemma~\ref{le:covariateShift}]
It holds that $P\bigl[P[A_i\,|\,\mathcal{F}] = 0,\, P[A_i\,|\,\mathcal{H}] > 0\bigr] =0$,
as shown in the proof of Proposition~\ref{pr:corrH}. This implies
\begin{equation*}
P\bigl[P[A_i\,|\,\mathcal{F}]= 0\bigr]\ \le\ P\bigl[P[A_i\,|\,\mathcal{H}] = 0\bigr]\ =\ 0.
\end{equation*}
Define $Z = \sum_{j=1}^\ell \frac{Q[A_j\,|\,\mathcal{F}]}{P[A_j\,|\,\mathcal{F}]}\,
    P[A_j\,|\,\mathcal{\mathcal{H}}]$. Then it holds that
\begin{equation}\label{eq:zero}
Z = 0 \quad\Rightarrow\quad \frac{Q[A_j\,|\,\mathcal{F}]}{P[A_j\,|\,\mathcal{F}]}\,
    P[A_j\,|\,\mathcal{\mathcal{H}}]\ = \ 0 \quad \text{for all}\ j = 1, \ldots, \ell.
\end{equation}
By Proposition~\ref{pr:corrH} and the assumption of covariate shift, we obtain on the
event $\{Z >0\}$ for all $i = 1, \ldots, \ell$
\begin{equation*}
P[A_i\,|\,\mathcal{H}]\ = \ \frac{Q[A_i\,|\,\mathcal{F}]}{P[A_i\,|\,\mathcal{F}]}\,
    \frac{P[A_i\,|\,\mathcal{H}]}{Z}.
\end{equation*}
By \eqref{eq:zero} and $P[A_i\,|\,\mathcal{H}] > 0$, it follows
that for all $i = 1, \ldots, \ell$ with probability $1$ under $P$
\begin{equation*}
Z \ = \ \frac{Q[A_i\,|\,\mathcal{F}]}{P[A_i\,|\,\mathcal{F}]}.
\end{equation*}
This implies $1 = E_P[Z\,|\,\mathcal{F}] = \frac{Q[A_i\,|\,\mathcal{F}]}{P[A_i\,|\,\mathcal{F}]}$ and hence
also \eqref{eq:hence}.
\end{proof}

\begin{proof}[\textbf{Proof of Theorem~\ref{th:SJSvsCovShift}}]
Show (i): Since $R$ is invertible by assumption, Lemma \ref{le:total} implies $Q[A_i\,|\,\mathcal{F}] =
P[A_i\,|\,\mathcal{F}]$ for all $i= 1, \ldots, \ell$. (CSH) now follows from
Proposition~\ref{pr:corrH}.

Show (ii): For any $i = 1, \ldots, \ell$ and $H \in \mathcal{H}$, it holds that
\begin{align}
Q[A_i \cap H\,|\,\mathcal{F}] & = E_Q\bigl[\mathbf{1}_H\,Q[A_i\,|\,\mathcal{H}]\,\big|\,\mathcal{F}\bigr]\notag\\
    & \stackrel{(CSH)}{=} E_Q\bigl[\mathbf{1}_H\,P[A_i\,|\,\mathcal{H}]\,\big|\,\mathcal{F}\bigr]\notag\\
    & \stackrel{(CDI)}{=} E_P\bigl[\mathbf{1}_H\,P[A_i\,|\,\mathcal{H}]\,\big|\,\mathcal{F}\bigr]\notag\\
    & = E_P\bigl[P[A_i\cap H\,|\,\mathcal{H}]\,\big|\,\mathcal{F}\bigr]\notag\\
    & = P[A_i\cap H\,|\,\mathcal{F}].\label{eq:msparsecov}
\end{align}
Eq.~\eqref{eq:msparsecov}, with $H = \Omega$, implies $Q[A_i\,|\,\mathcal{F}] = P[A_i\,|\,\mathcal{F}]$,
$i=1, \ldots, \ell$. Dividing both sides of \eqref{eq:msparsecov} by $Q[A_i\,|\,\mathcal{F}]$ then
gives (with the convention `$0 / 0 = 0$')
\begin{equation*}
Q_i[H\,|\,\mathcal{F}] = \frac{Q[A_i \cap H\,|\,\mathcal{F}]}{Q[A_i\,|\,\mathcal{F}]} =
\frac{P[A_i \cap H\,|\,\mathcal{F}]}{P[A_i\,|\,\mathcal{F}]} = P_i[H\,|\,\mathcal{F}],
\end{equation*}
i.e.\ $Q$  is related to $P$ through $\mathcal{F}$-SJS.

Show (iii): If (SJS) and (CSH) as well as \eqref{eq:positive} are true then Lemma~\ref{le:covariateShift}
is applicable such that \eqref{eq:hence} follows.
As a consequence,
we can make the following calculation that applies to all $H\in\mathcal{H}$:
\begin{align*}
Q[H\,|\,\mathcal{F}] & = E_Q\bigl[Q[H\,|\,\sigma(\mathcal{F}\cup\mathcal{A})]\,\big|\,\mathcal{F}\bigr]\\
    & = \sum_{i=1}^\ell E_Q\bigl[Q_i[H\,|\,\mathcal{F}]\,\mathbf{1}_{A_i}\,\big|\,\mathcal{F}\bigr]\\
    & \stackrel{(SJS)}{=} \sum_{i=1}^\ell E_Q\bigl[P_i[H\,|\,\mathcal{F}]\,\mathbf{1}_{A_i}\,\big|\,\mathcal{F}\bigr]\\
    & = \sum_{i=1}^\ell P_i[H\,|\,\mathcal{F}]\,Q[A_i\,|\,\mathcal{F}]\\
    & \stackrel{\eqref{eq:hence}}{=}
        \sum_{i=1}^\ell P_i[H\,|\,\mathcal{F}]\,P[A_i\,|\,\mathcal{F}]\\
    & = \sum_{i=1}^\ell P[A_i\cap H\,|\,\mathcal{F}]\\
    & = P[H\,|\,\mathcal{F}].
\end{align*}
Thus, we have shown that (CDI) holds true.
\end{proof}

\section{Estimating sparse joint shift when all features are discrete}
\label{se:instant}

We consider here the special case of Assumption~\ref{as:setting} where the feature information set $\mathcal{H}$
is generated by an at most countable set of events. This implies that also any sub-$\sigma$-algebra $\mathcal{F}$
of $\mathcal{H}$
is generated by an at most countable set of events. Think of $\mathcal{H} = \sigma(X_i, \, i \in I)$ for
a feature vector $(X_i, \, i \in I)$ taking values in an at most countable set and
$\mathcal{F} = \sigma(X_j, \, j \in J)$ for a non-empty $J \subset I$.

For the following, let Assumption~\ref{as:setting} hold true and suppose that
\begin{itemize}
\item $\mathcal{H}$ is generated by an at most countable set of events, and
\item $\mathcal{F}$ is a sub-$\sigma$-algebra of $\mathcal{H}$ such that $Q$ is related to $P$ through
    $\mathcal{F}$-SJS in the sense of Definition~\ref{de:SJS}.
\end{itemize}
It is convenient to begin the description of this case with the sub-$\sigma$-algebra $\mathcal{F}$:
\begin{equation}\label{eq:Fdiscreted}
\mathcal{F}  = \sigma\bigl(\{F_n: n \in \mathbb{N}\}\bigr),
\end{equation}
for an $\mathcal{M}$-measurable disjoint decomposition $F_n$, $n \in \mathbb{N}$, of $\Omega$.
In practice, all but finitely many of the $F_n$ will be the $\emptyset$.
Because of $\mathcal{H} \supset \mathcal{F}$, \eqref{eq:Fdiscreted} implies for $\mathcal{H}$
\begin{equation*}
\mathcal{H}   = \sigma\bigl(\{H_{n, r}: n \in \mathbb{N}, r = 1,\ldots, r_n\}\bigr), \qquad
\text{with} \ \bigcup_{r=1}^{r_n} H_{n, r} = F_n, \quad \text{for all}\ n \in \mathbb{N},
\end{equation*}
where for each $n\in\mathbb{N}$, the set family $H_{n, 1}, \ldots, H_{n, r_n}$ is an $\mathcal{M}$-measurable
disjoint decomposition of $F_n$. We denote by $I_\mathcal{H} = \{(n,r): n\in \mathbb{N}, r = 1, \ldots, r_n\}$
the index set of $\mathcal{H}$.

Having introduced the basic notation, we obtain the following formulae for the quantities
needed for the description of the model:
\begin{itemize}
\item Posterior probabilities conditional on $\mathcal{H}$, for $i = 1, \ldots, \ell$:
\begin{equation*}
P[A_i\,|\,\mathcal{H}] \ = \ \sum_{(n,r)\in I_\mathcal{H}} \mathbf{1}_{H_{n,r}}\,P[A_i\,|\,H_{n,r}].
\end{equation*}
\item Density $h_i$, $i = 1, \ldots, \ell$, of $Q_i$ with respect to $P_i$ on $\mathcal{H}$,
without making use of the SJS assumption:
\begin{equation*}
h_i \ = \ \sum_{(n,r)\in I_\mathcal{H}} \mathbf{1}_{H_{n,r}}\,\frac{Q[H_{n,r}\,|\,A_i]}{P[H_{n,r}\,|\,A_i]}
    \ = \ \sum_{(n,r)\in I_\mathcal{H}} \mathbf{1}_{H_{n,r}}\,\frac{Q_i[H_{n,r}]}{P_i[H_{n,r}]}.
\end{equation*}
\item Density $f_i$, $i = 1, \ldots, \ell$, of $Q_i$ with respect to $P_i$ on $\mathcal{H}$ under
the SJS assumption:
\begin{equation*}
f_i \ = \ \sum_{n\in\mathbb{N}} \mathbf{1}_{F_N}\,\frac{Q_i[F_n]}{P_i[F_n]} \ = \
    \sum_{n\in\mathbb{N}} \mathbf{1}_{F_N}\,E_{P_i}[h_i\,|\,F_n].
\end{equation*}
\item Density $h$ of $Q$ with respect to $P$ on $\mathcal{H}$,
without making use of the SJS assumption:
\begin{equation*}
h \ = \ \sum_{(n,r)\in I_\mathcal{H}} \mathbf{1}_{H_{n,r}}\,\frac{Q[H_{n,r}]}{P[H_{n,r}]}.
\end{equation*}
\item $h_Q$ from \eqref{eq:Hdens} in the discrete setting (i.e.\
the density $h_Q$ of $Q$ with respect to $P$ on $\mathcal{H}$,
explicitly making use of the SJS assumption):
\begin{align*}
h_Q & = \sum_{i=1}^\ell \left(\sum_{n\in\mathbb{N}} \mathbf{1}_{F_N}\,\frac{Q_i[F_n]}{P_i[F_n]}\right)
    \frac{Q[A_i]}{P[A_i]} \Big(\sum_{(n,r)\in I_\mathcal{H}} \mathbf{1}_{H_{n,r}}\,P[A_i\,|\,H_{n,r}]\Big)\\
    & = \sum_{i=1}^\ell \sum_{n\in\mathbb{N}} \frac{Q[F_n\cap A_i]}{P[F_n\cap A_i]}
    \left(\sum_{r=1}^{r_n} \mathbf{1}_{H_{n,r}}\,P[A_i\,|\,H_{n,r}]\right).
\end{align*}
\item Difference of $h$ and $h_Q$:
\begin{equation}\label{eq:discreteDiff}
h - h_Q \ = \ \sum_{(n,r)\in I_\mathcal{H}} \mathbf{1}_{H_{n,r}}
    \left(\frac{Q[H_{n,r}]}{P[H_{n,r}]} -  \sum_{i=1}^\ell \frac{Q[F_n\cap A_i]}{P[F_n\cap A_i]}\,
    P[A_i\,|\,H_{n,r}]\right).
\end{equation}
\end{itemize}
Consequence of \eqref{eq:discreteDiff}: On the event $F_n$, it holds that
\begin{equation*}
(h - h_Q)\,\mathbf{1}_{F_n} \ = \ \sum_{r=1}^{r_n} \mathbf{1}_{H_{n,r}}
    \left(\frac{Q[H_{n,r}]}{P[H_{n,r}]} -  \sum_{i=1}^\ell Q[F_n\cap A_i]\,
    \frac{P[A_i\,|\,H_{n,r}]}{P[F_n\cap A_i]}\right).
\end{equation*}
Hence, for each $n\in\mathbb{N}$, the unobservable probabilities $Q[F_n\cap A_1], \ldots, Q[F_n\cap A_\ell]$
must satisfy the following system of $r_n$ linear equations for $\ell$ unknowns:
\begin{equation}\label{eq:linear}
\sum_{i=1}^\ell Q[F_n\cap A_i]\,\frac{P[A_i\,|\,H_{n,r}]}{P[F_n\cap A_i]} \ = \
    \frac{Q[H_{n,r}]}{P[H_{n,r}]}, \qquad r = 1, \ldots, r_n.
\end{equation}
\begin{subequations}
Note that once the $Q[F_n\cap A_i]$ have been determined for all $n\in\mathbb{N}$,
the target class priors are obtained by summing them up:
\begin{equation}\label{eq:confusion}
Q[A_i] \ = \ \sum_{n\in\mathbb{N}} Q[F_n\cap A_i], \qquad i = 1, \ldots, \ell.
\end{equation}
Moreover, evoking Proposition~\ref{pr:corrH} gives for all $(n,r) \in I_\mathcal{H}$ and
$i = 1, \ldots, \ell$ the \emph{conditional posterior correction formula} in the discrete case:
\begin{equation}
Q[A_i\,|\,H_{n,r}] \ = \ \frac{\frac{Q[A_i\cap F_n]}{P[A_i\cap F_n]}
    \,P[A_i\,|\,H_{n,r}]}{\sum_{j=1}^\ell \frac{Q[A_j\cap F_n]}{P[A_j\cap F_n]}\,P[A_j\,|\,H_{n,r}]}.
\end{equation}
\end{subequations}
As alluded to in Section~\ref{se:discrete} above,
in order to have any chance of uniqueness of the solution to \eqref{eq:linear} (for fixed $n\in\mathbb{N}$),
the number of equations ought to be at least equal to or greater than the number of unknowns. Hence $r_n \ge \ell$
is a natural requirement. Depending on the nature of the features, this requirement might not be satisfied, in
particular if the original feature information set $\mathcal{H}$ is replaced by some subset $\mathcal{H}'
\supset \mathcal{F}$ for efficiency reasons like the one that problem \eqref{eq:optimH} has to be solved instead of
\eqref{eq:optim}.

Chen et al.~\cite{chen&zaharia&Zou:SJS} deal with this issue by introducing a (hard) classifier $\mathbf{C} =
(C_1, \ldots, C_\ell)$ as in \eqref{eq:classifier}.
They then consider an augmented sub-$\sigma$-algebra $\mathcal{H}_\mathcal{C} = \sigma(\mathcal{H}'\cup\mathcal{C})$
of $\mathcal{H}$. Substituting this in \eqref{eq:linear} produces the equation system
\begin{equation}\label{eq:linearC}
\sum_{i=1}^\ell Q[F_n\cap A_i]\,\frac{P[A_i\,|\,H_{n,r}\cap C_j]}{P[F_n\cap A_i]} \ = \
    \frac{Q[H_{n,r}\cap C_j]}{P[H_{n,r}\cap C_j]}, \qquad r = 1, \ldots, r_n, j = 1, \ldots, \ell.
\end{equation}

The special choice $\mathcal{H}_\mathcal{C} = \sigma(\mathcal{F}\cup\mathcal{C})$ in \eqref{eq:linearC}
gives a conditional version of the confusion matrix approach to class prior estimation
(Saerens et al.~\cite{saerens2002adjusting}) in the discrete setting:
\begin{equation}\label{eq:linearF}
\sum_{i=1}^\ell Q[F_n\cap A_i]\,\frac{P[A_i\,|\,F_n\cap C_j]}{P[F_n\cap A_i]} \ = \
    \frac{Q[F_n\cap C_j]}{P[F_n\cap C_j]}, \qquad j = 1, \ldots, \ell.
\end{equation}
By combining \eqref{eq:confusion} and \eqref{eq:linearF}, a two-step approach to class prior estimation
under (proper) sparse joint shift can be implemented.

\providecommand{\href}[2]{#2}
\providecommand{\arxiv}[1]{\href{http://arxiv.org/abs/#1}{arXiv:#1}}
\providecommand{\url}[1]{\texttt{#1}}
\providecommand{\urlprefix}{URL }

\medskip
Received October 2023; revised February 2024; early access March 2024.
\medskip

\end{document}